\documentclass[lettersize,journal]{IEEEtran}
\usepackage{amsmath,amsfonts}
\usepackage{algorithmic}
\usepackage{algorithm}
\usepackage{array}
\usepackage[caption=false,font=normalsize,labelfont=sf,textfont=sf]{subfig}
\usepackage{textcomp}
\usepackage{stfloats}
\usepackage{url}
\usepackage{verbatim}
\usepackage{graphicx}
\usepackage{cite}
\usepackage{xcolor}
\usepackage{colortbl}
\usepackage{arydshln}       

\usepackage{amsthm}
\usepackage{booktabs}
\usepackage{multirow}
\theoremstyle{plain}     
\newtheorem{theorem}{Theorem}    
\newtheorem{corollary}{Corollary}
\newtheorem{assumption}{Assumption}
\newtheorem{definition}[theorem]{Definition}

\usepackage{makecell}

\usepackage{algorithm}
\usepackage{algorithmic}

\usepackage{subcaption}     
\usepackage{wrapfig}        
\usepackage{float}          

\newtheorem{proposition}{Proposition}

\definecolor{mylightgray}{gray}{0.9}
\definecolor{mygreen}{RGB}{0,150,0}
\usepackage[colorlinks=true, linkcolor=blue, citecolor=blue]{hyperref}

\begin{document}

\title{Learning Transferable Topology Priors for Multi-Agent LLM Collaboration Across Domains}


\author{
    Taolin Zhang$^{1}$,
    Zijie Zhou$^2$,
    Jiuheng Wan$^{1}$,
    Tingyuan Hu$^3$,
    Chengyu Wang$^4$,
    Xiaofeng He$^3$ \IEEEmembership{Member, IEEE},
    and Richang Hong$^1$ \IEEEmembership{Senior Member, IEEE}
\thanks{
    $^1$ Taolin Zhang, Jiuheng Wan, and Richang Hong are with the Hefei University of Technology, Hefei 230002, China (e-mail: tlzhang@hfut.edu.cn; wan\_jiuheng@163.com; hongrc@hfut.edu.cn).\\
    $^2$ Zijie Zhou is with the China University of Petroleum (Beijing), Beijing 102249, China (e-mail: zjzhouzh@gmail.com).\\
    $^3$ Tingyuan Hu and Xiaofeng He are with the East China Normal University, Shanghai 200062, China (e-mail: 10245102409@stu.ecnu.edu.cn; hexf@cs.ecnu.edu.cn).\\
    $^4$ Chengyu Wang is with the Alibaba Group, Hangzhou 310052, China (e-mail: chengyu.wcy@alibaba-inc.com). \\
    Corresponding authors: Chengyu Wang and Richang Hong.}
}

\markboth{IEEE TRANSACTIONS ON KNOWLEDGE AND DATA ENGINEERING, VOL. XX, NO. XX, XXXX 202X}%
{Shell \MakeLowercase{\textit{et al.}}: A Sample Article Using IEEEtran.cls for IEEE Journals}




\maketitle

\begin{abstract}
Large language model (LLM)-based multi-agent systems have shown strong potential for complex reasoning by coordinating specialized agents through structured communication. However, existing topology-evolution methods typically construct or optimize a collaboration topology for each query from scratch, leading to substantial online search overhead, high inference-time token consumption, and limited scalability in multi-domain settings.
We propose \texttt{TopoPrior}, a framework for learning transferable topology priors for multi-agent LLM collaboration across domains. Rather than repeatedly searching for effective collaboration structures online, \texttt{TopoPrior} learns reusable topology priors from reference collaboration graphs collected offline from multiple domains and uses them to generate query-conditioned initial collaboration graphs for downstream refinement. By shifting part of topology search from per-query online optimization to offline prior learning, \texttt{TopoPrior} amortizes search cost while remaining compatible with existing topology-evolution backbones.
Technically, \texttt{TopoPrior} contains two key components. First, a \textbf{transferable topology prior learning} module employs a conditional variational graph framework to capture reusable structural regularities across domains in a latent space. Second, a \textbf{query-conditioned latent adaptation} module introduces adversarial alignment to reduce unnecessary domain discrepancy while preserving query-relevant structural variation.
Experiments on multi-domain reasoning benchmarks show that \texttt{TopoPrior} consistently improves several heterogeneous topology-evolution backbones while reducing online inference-time token usage, with only modest additional trainable parameters. These results suggest that transferable topology initialization is an effective and lightweight mechanism for improving the efficiency of multi-agent LLM collaboration across domains.
\end{abstract}


\section{Introduction}
\IEEEPARstart{L}{arge} language model (LLM)-based multi-agent systems have emerged as a promising paradigm for complex reasoning by coordinating multiple specialized agents through structured communication. This paradigm is particularly appealing in multi-domain settings, such as healthcare~\cite{medical_tc}, science~\cite{DBLP:conf/iclr/HendrycksBBZMSS21,DBLP:conf/nips/HuangBZZZSLLZLF23}, and law~\cite{DBLP:conf/lrec/NghiemBFA22,DBLP:journals/corr/abs-2510-08524}, where different queries may require different combinations of expertise and collaboration patterns. A central challenge in this setting is how to \emph{reuse} effective collaboration structures across domains, so that multi-agent systems can adapt to new tasks without repeatedly searching for communication topologies from scratch.

Existing approaches to multi-domain LLM reasoning can be broadly grouped into three categories, each with different trade-offs in flexibility, effectiveness, and computational cost. \textit{Training-free prompting} methods adapt frozen LLMs through prompt engineering, few-shot demonstrations, and chain-of-thought reasoning~\cite{DBLP:conf/nips/ChenHWJ23,DBLP:conf/emnlp/NguyenLZZY23,DBLP:conf/emnlp/Hu0YXW024}. Although easy to deploy, their performance is ultimately limited by the capability of the underlying model~\cite{DBLP:conf/acl/ZhangWQLZHXLY025}. \textit{Fine-tuning}-based methods specialize models to target domains through parameter updates~\cite{DBLP:conf/emnlp/ScialomCM22,DBLP:journals/corr/abs-2502-04380}, but they may suffer from catastrophic forgetting and introduce substantial training and storage overhead when many domains are involved~\cite{DBLP:journals/pami/LiYDZS24,DBLP:conf/icml/Liang00LWCH25}. By contrast, \textit{multi-agent collaboration} improves reasoning by decomposing tasks across specialized agents and organizing their interactions through communication topologies. Recent work has progressed from static topologies to dynamic graph construction, including reinforcement-learned, pruning-based, and autoregressive approaches~\cite{DBLP:conf/iclr/HongZCZCWZWYLZR24,liu2024,DBLP:conf/iclr/ZhangYLYWWCY025,DBLP:conf/icml/ZhangYSWYF00C25}. However, most of these methods treat topology construction as a query-level optimization problem and often perform search from scratch, which can incur substantial online overhead, accumulated communication noise, and high inference-time token consumption in multi-domain settings~\cite{li-etal-2025-multi-agent}; see Fig.~\ref{motivation_fig}.

\begin{figure*}[t!]
\centering
\includegraphics[width=17cm]{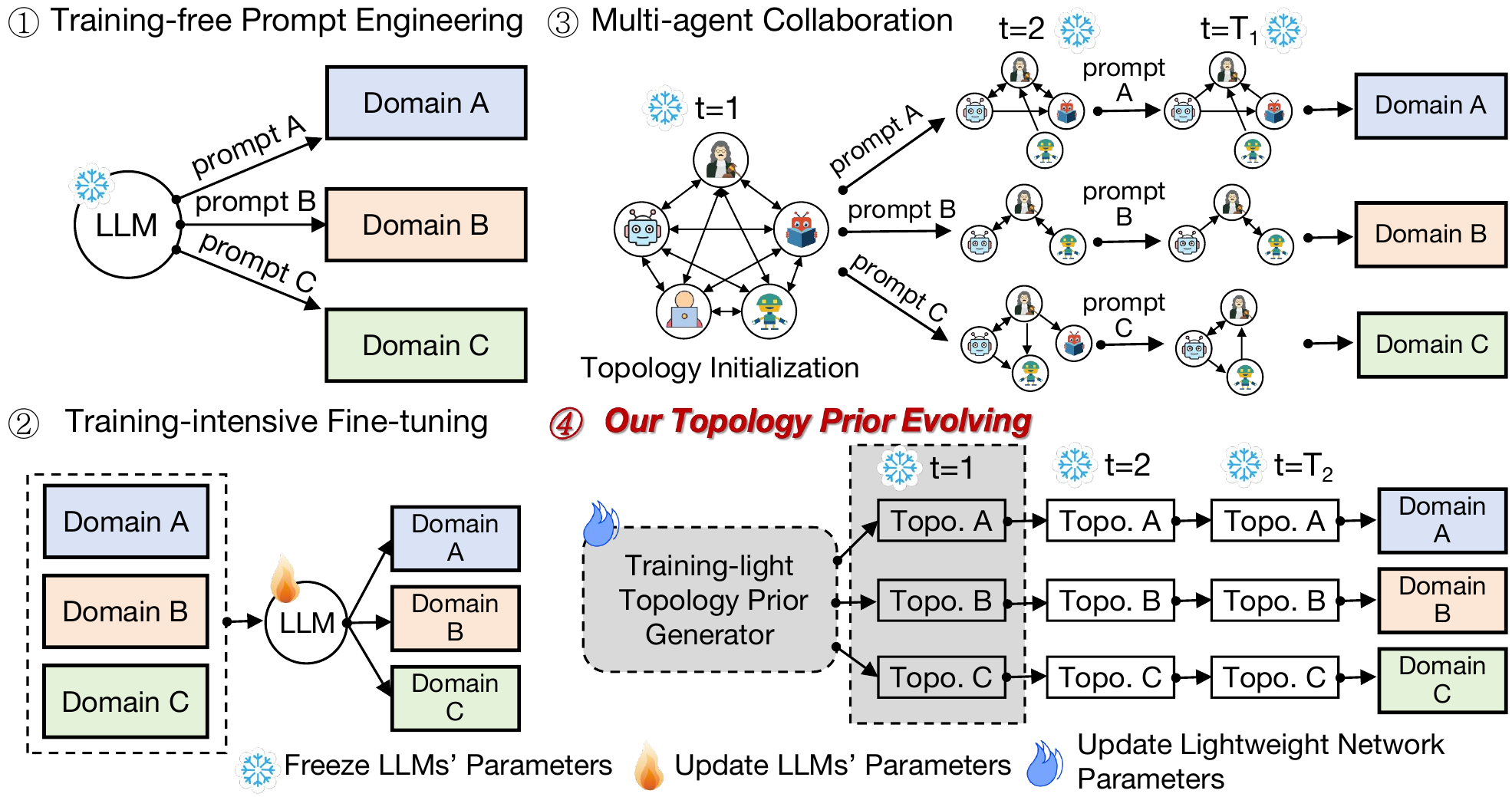}
\caption{Comparison of reasoning paradigms across domains.
\textbf{(1) Training-free methods} rely on frozen LLMs and are ultimately limited by the intrinsic capability of the underlying model.
\textbf{(2) Training-intensive methods} update model parameters, but may incur catastrophic forgetting and substantial computational overhead across domains.
\textbf{(3) Multi-agent methods} dynamically construct collaboration topologies, yet repeatedly optimizing graphs from scratch for each query can incur high online token cost.
Our method learns transferable topology priors that provide query-aware initialization for downstream topology evolution.}
\label{motivation_fig}
\vspace{-1.5em}
\end{figure*}

This limitation motivates a different perspective: \textit{instead of treating every query as a new topology-search problem, can we learn reusable collaboration patterns offline and use them to initialize downstream topology evolution?} To this end, we propose \texttt{TopoPrior}, a framework for learning transferable topology priors for multi-agent LLM collaboration across domains. The key idea is to shift part of the topology-search burden from per-query online graph construction to offline cross-domain prior learning. Rather than evolving a collaboration graph from scratch for every incoming query, \texttt{TopoPrior} learns reusable topology priors from reference collaboration graphs collected across multiple domains, and then uses these priors to generate query-conditioned initial collaboration graphs for downstream refinement. In this way, \texttt{TopoPrior} amortizes part of the topology-search cost across queries and domains while remaining compatible with existing topology-evolution backbones.

\texttt{TopoPrior} contains two key components.
\textbf{(1) Transferable Topology Prior Learning.}
We develop a conditional variational graph framework to capture reusable structural regularities in collaboration graphs across domains. Specifically, a variational encoder maps collaboration graphs and their associated queries into a latent space, while a conditional generator reconstructs query-conditioned initial topologies from the learned prior and the input query. This design enables \texttt{TopoPrior} to model collaboration structures that can be reused across related domains, rather than relearned independently for each query.
\textbf{(2) Query-Conditioned Latent Adaptation.}
While transferable priors can improve initialization efficiency, effective collaboration still requires sensitivity to query-relevant structural variation. We therefore introduce an adversarial latent adaptation module to reduce unnecessary domain discrepancy while preserving query-dependent structural characteristics. As a result, the generated topologies reflect both reusable collaboration regularities and task-specific specialization.

We evaluate \texttt{TopoPrior} on MMLU~\cite{DBLP:conf/iclr/HendrycksBBZMSS21} and C-Eval~\cite{DBLP:conf/nips/HuangBZZZSLLZLF23} under multi-domain settings. Experimental results show that integrating \texttt{TopoPrior} with existing topology-evolution backbones improves downstream performance in most evaluated settings, reduces online inference-time token usage by up to \textbf{40.22\%}, and introduces only \textbf{3.3M} additional trainable parameters for the topology-prior generator. These results suggest that transferable topology initialization can serve as an effective and lightweight mechanism for improving the efficiency of multi-agent LLM collaboration across domains.

The main contributions of this work are threefold:
\begin{itemize}
    \item We formulate topology initialization for multi-agent LLM systems in multi-domain settings as a transferable topology prior learning problem, in which reusable collaboration structures are learned offline and reused to improve downstream topology evolution.
    \item We propose \texttt{TopoPrior}, a lightweight framework that combines conditional variational topology prior learning with query-conditioned latent adaptation to generate informative initial collaboration graphs.
    \item We conduct experiments on MMLU and C-Eval with multiple topology-evolution backbones, showing that \texttt{TopoPrior} improves downstream reasoning performance and online inference token efficiency across the evaluated settings with modest parameter overhead.
\end{itemize}

\section{Related Work}

\subsection{Multi-Domain Reasoning with LLMs}
Reasoning across multiple domains has become an increasingly important topic with the rise of LLMs~\cite{DBLP:conf/acl/KwanZ0SLJS0W24,DBLP:conf/acl/Chen0ZC0C24,DBLP:conf/acl/0004SSDYGSANMP25}. Existing approaches mainly follow two paradigms. \emph{Training-free methods} adapt frozen LLMs through in-context learning, prompt engineering, and chain-of-thought reasoning~\cite{DBLP:conf/nips/ChenHWJ23,multi_domain_kt,DBLP:conf/emnlp/NguyenLZZY23,DBLP:conf/emnlp/Hu0YXW024}. They are easy to deploy and avoid parameter updates, but their performance is ultimately limited by the capability of the underlying model. \emph{Training-intensive methods} adapt LLMs through supervised fine-tuning or parameter-efficient tuning such as LoRA~\cite{DBLP:conf/iclr/HuSWALWWC22}. Although effective for domain specialization, they may suffer from catastrophic forgetting and require additional storage and maintenance when many domains are involved~\cite{DBLP:journals/corr/abs-2410-10181,DBLP:journals/corr/abs-2509-16882,DBLP:conf/icml/Liang00LWCH25,DBLP:journals/pami/LiYDZS24}. In contrast to prompt-level or parameter-level adaptation, our work studies a different axis of adaptation, namely structural adaptation at the level of inter-agent communication. Specifically, rather than modifying prompts or updating LLM parameters, we improve multi-domain reasoning efficiency by learning reusable topology priors for multi-agent collaboration.

\subsection{Topology Design in Multi-Agent LLM Systems}
Multi-agent LLM systems improve reasoning by organizing agent interactions through communication topologies. Early methods mainly adopt predefined interaction patterns, including independent aggregation~\cite{DBLP:conf/acl/Jiang0L23,DBLP:conf/acl/ZhangX0LHD24,DBLP:conf/icml/Du00TM24}, chain-based communication~\cite{DBLP:journals/corr/abs-2310-02003,DBLP:conf/acl/QianLLCDL0CSCXL24,DBLP:conf/iclr/HongZCZCWZWYLZR24}, star-style coordination~\cite{wu2024autogen,DBLP:journals/corr/abs-2311-13373}, and tree-structured hierarchies~\cite{DBLP:journals/corr/abs-2404-02183}. While effective, these fixed designs are often inflexible when task requirements vary across inputs or domains.

More recent studies dynamically construct or optimize collaboration graphs. GPTSwarm~\cite{DBLP:conf/icml/ZhugeWKFKS24} and DyLAN~\cite{liu2024} adapt agent interactions through reinforcement learning or dynamic agent selection. Pruning-based methods remove redundant edges or agents to obtain task-adaptive sparse graphs~\cite{DBLP:conf/iclr/ZhangYLYWWCY025,DBLP:conf/acl/WangW00Z0025}, while autoregressive approaches generate collaboration structures conditioned on the input query~\cite{DBLP:conf/icml/ZhangYSWYF00C25,DBLP:conf/aaai/LiLWZP26}. Although these methods have shown strong reasoning performance, most of them construct or optimize topologies at query time and often search from scratch, incurring substantial online search cost and inference-time token overhead. Our work is complementary to this line of research: rather than replacing downstream topology evolution, we learn transferable topology priors that provide stronger initialization for it.

\subsection{Transferable Structure Learning and Graph Initialization}
Reusing structural knowledge across tasks or domains is important for efficient adaptation. Prior work in domain generalization and representation transfer suggests that shared latent structure can support transfer across heterogeneous settings~\cite{dua-etal-2023-adapt,DBLP:journals/pami/LiYDZS24,dey-lal-2025-transferability}. In graph representation and generation learning, variational and conditional generative frameworks have been used to model reusable structural patterns from observed graphs~\cite{DBLP:journals/corr/KipfW16a,cao-etal-2025-graphinsight}. More broadly, warm-start initialization and structure reuse have long been recognized as practical strategies for reducing optimization cost in complex search problems \cite{DBLP:conf/wacv/GoswamiSWS23,DBLP:conf/icml/AngellM24}.
However, existing multi-agent LLM methods treat topology construction as a query-level search problem without learning transferable priors across domains, thus failing to amortize reusable collaboration patterns and keeping graph construction query-specific.

In contrast, our work introduces two key differences: (1) learning a \emph{transferable topology prior} for graph initialization instead of searching from scratch per query; (2) combining prior learning with query-conditioned latent adaptation to preserve both cross-domain regularities and query-specific specialization. This design amortizes topology-search cost across domains and enhances downstream efficiency.

\section{TopoPrior: Framework}
In this section, we formalize the problem of transferable topology prior learning for multi-agent LLM collaboration in multi-domain settings and present the proposed \texttt{TopoPrior} framework. An overview of \texttt{TopoPrior} is shown in Fig.~\ref{model_fig}.

\subsection{Problem Definition}
Assume that we are given training data from $K$ domains, denoted by $\mathcal{D}=\bigcup_{k=1}^{K}\mathcal{D}^{k}$, where $\mathcal{D}^{k}=\{(q_i^{k},\mathbb{G}_i^{k},y_i^{k})\}_{i=1}^{M_k}$. Here, $q_i^{k}$ is a query from the $k$-th domain, $y_i^{k}$ is the corresponding ground-truth answer, and $M_k$ is the number of samples in that domain. $\mathbb{G}_i^{k}=(\mathcal{V}_i^{k},\mathcal{E}_i^{k})$ denotes a reference collaboration graph produced by a strong topology-evolution method.\footnote{In our implementation, we use AgentDropout~\cite{DBLP:conf/acl/WangW00Z0025} to construct reference collaboration graphs.} The reference graph is not assumed to be globally optimal; rather, it serves as an effective but imperfect source of structural supervision for learning transferable topology priors across domains.

Let $\mathcal{A}_i^{k}\in\mathbb{R}^{N\times N}$ denote the adjacency matrix derived from $\mathbb{G}_i^{k}$, where $N$ is the size of the candidate role pool and serves as an upper bound on the number of generated agent nodes. The goal of \texttt{TopoPrior} is to learn a query-conditioned topology initializer that maps a new query $q$ to an initial collaboration graph $\hat{\mathbb{G}}$, such that $\hat{\mathbb{G}}$ provides an informative starting point for downstream topology evolution. At inference time, the learned initializer is integrated with an existing topology-evolution backbone, which further refines $\hat{\mathbb{G}}$ according to its original task-specific optimization process. In our framework, topology-prior learning primarily uses query--graph pairs $(q_i^{k},\mathbb{G}_i^{k})$, while the task labels $y_i^{k}$ are retained for reference-graph construction and downstream task evaluation.

\begin{figure*}[!t]
\centering
\includegraphics[width=17cm]{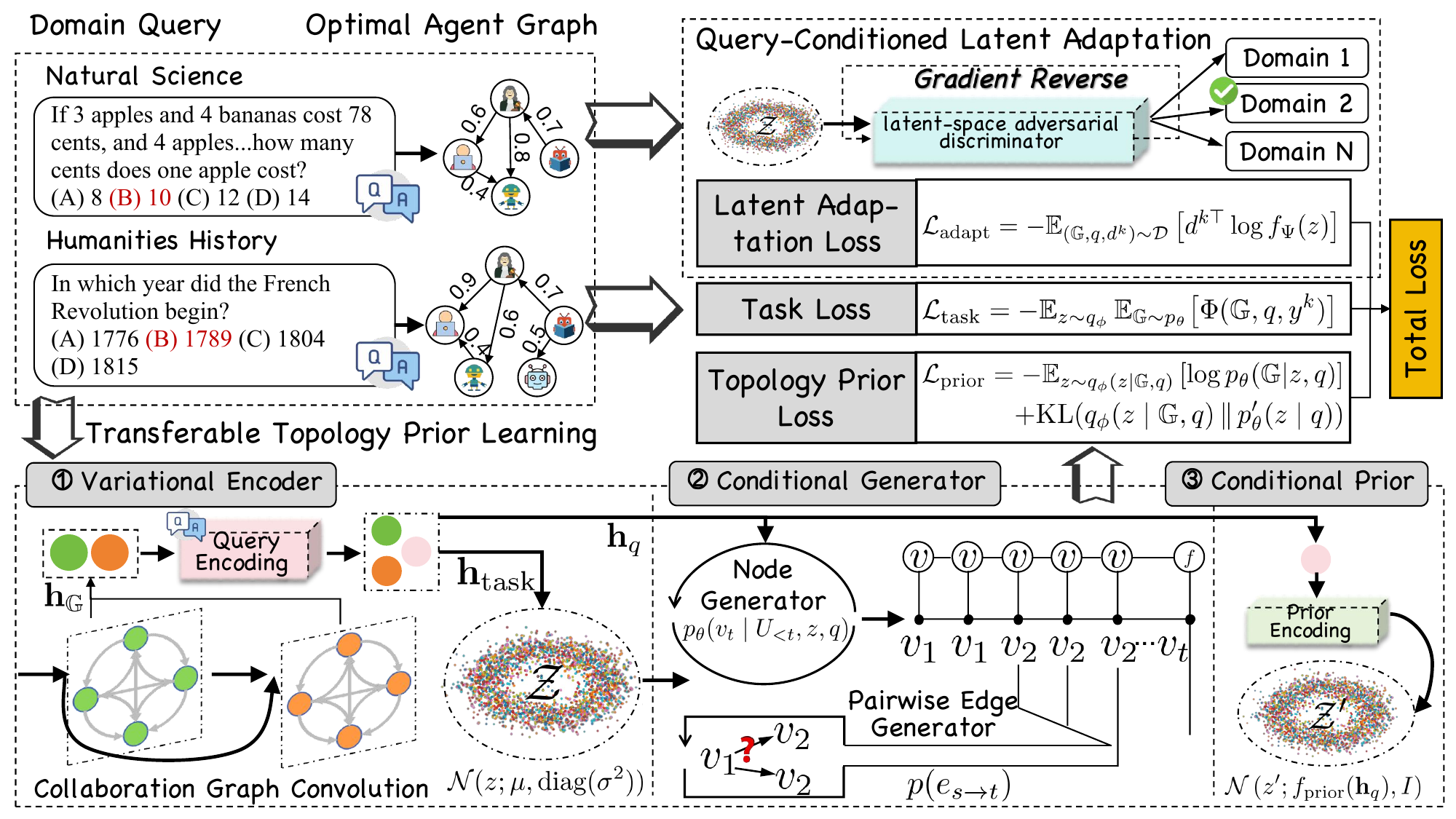}
\caption{Overview of \texttt{TopoPrior}. (1) \emph{Transferable Topology Prior Learning} captures reusable collaboration patterns from multiple domains through a conditional variational graph framework. (2) \emph{Query-Conditioned Latent Adaptation} improves cross-domain robustness by adversarially regularizing the latent space while retaining query-relevant structural information.}
\label{model_fig}
\vspace{-1.4em}
\end{figure*}

\subsection{Transferable Topology Prior Learning}
The first component of \texttt{TopoPrior} aims to learn reusable topology priors that can initialize collaboration graphs across domains. Since queries from different domains may induce distinct yet structurally related collaboration patterns, we adopt a conditional variational graph framework~\cite{DBLP:journals/corr/KipfW16a} to encode reference graphs and queries into a shared latent space and reconstruct query-conditioned initial topologies from the resulting latent prior.

Formally, given a query--graph pair $(q,\mathbb{G})$, we model the conditional likelihood of the collaboration graph as
\begin{multline}
\log p_\theta(\mathbb{G}\mid q)
\geq
\mathbb{E}_{z\sim q_\phi(z\mid \mathbb{G},q)}
\big[\log p_\theta(\mathbb{G}\mid z,q)\big] \\
-
\mathrm{KL}\!\left(q_\phi(z\mid \mathbb{G},q)\,\|\,p'_\theta(z\mid q)\right),
\end{multline}
where $q_\phi(z\mid \mathbb{G},q)$ is the variational encoder, $p_\theta(\mathbb{G}\mid z,q)$ is the conditional graph generator, and $p'_\theta(z\mid q)$ is a query-conditioned prior. This formulation allows \texttt{TopoPrior} to capture reusable structural regularities from reference collaboration graphs while preserving query-specific adaptation through the conditional prior.

\noindent\textbf{Variational Encoder.}
Given a query and its reference graph $(q,\mathbb{G})$, the encoder $q_\phi(z\mid \mathbb{G},q)$ uses a Graph Convolutional Network (GCN)~\cite{DBLP:conf/iclr/KipfW17} to encode the collaboration graph \footnote{All other alternative graph modeling methods can serve as replacements for this practice.}:
\begin{equation}
\mathbf{h}_{v}^{(l+1)}
=
f_\sigma\!\left(
\tilde{\mathcal{D}}^{-\frac{1}{2}}
\tilde{\mathcal{A}}_i^{k}
\tilde{\mathcal{D}}^{-\frac{1}{2}}
\mathbf{h}_{v}^{(l)}
\mathbf{W}_{\mathrm{ve}}^{(l)}
\right),
\end{equation}
where $\tilde{\mathcal{A}}_i^{k}=\mathcal{A}_i^{k}+I$ is the adjacency matrix with self-loops, $\tilde{\mathcal{D}}$ is the corresponding degree matrix, and $\mathbf{W}_{\mathrm{ve}}^{(l)}$ are trainable weights. The query representation $\mathbf{h}_q$ and the initial agent-role representation $\mathbf{h}_{v}^{(0)}$ are obtained from a frozen text encoder such as BERT \cite{DBLP:conf/naacl/DevlinCLT19}. We then compute a graph-level representation by sum pooling over node embeddings:
\begin{equation}
\mathbf{h}_{\mathbb{G}}=\sum_{n=1}^{N}\mathbf{h}_{v_n}.
\end{equation}
The query and graph representations are fused as
\begin{equation}
\mathbf{h}_{\mathrm{task}}=\mathrm{MLP}_{\mathrm{task}}(\mathbf{h}_q \parallel \mathbf{h}_{\mathbb{G}}).
\end{equation}
The latent variable $z$ is sampled from the Gaussian posterior:
\begin{gather}
\boldsymbol{\mu}=\mathbf{W}_{\mu}\mathbf{h}_{\mathrm{task}}+\mathbf{b}_{\mu},\\
\log \boldsymbol{\sigma}^{2}=\mathbf{W}_{\sigma}\mathbf{h}_{\mathrm{task}}+\mathbf{b}_{\sigma},\\
q_\phi(z\mid \mathbb{G},q)=\mathcal{N}\!\left(z;\boldsymbol{\mu},\mathrm{diag}(\boldsymbol{\sigma}^{2})\right),
\end{gather}
where $z=\boldsymbol{\mu}+\boldsymbol{\sigma}\odot \varepsilon$, $\varepsilon\sim\mathcal{N}(0,I)$, and $\boldsymbol{\sigma}=\exp\!\left(\frac{1}{2}\log \boldsymbol{\sigma}^{2}\right)$.

\noindent\textbf{Conditional Generator.}
The conditional generator reconstructs an initial collaboration graph autoregressively:
\begin{multline}
p_\theta(\mathbb{G}\mid z,q)
=
\prod_{t=1}^{T} p_\theta(v_t\mid U_{<t},z,q) \\
\times
\prod_{t=2}^{T}\prod_{s=1}^{t-1}
p_\theta(e_{s\rightarrow t}\mid v_s,v_t,U_{<t},z,q),
\end{multline}
where $U_{<t}$ denotes the set of previously generated nodes, and $T$ is the total number of generation steps.
$p_\theta(e_{s\rightarrow t})$ represents the generation probability of edges between nodes. If it is greater than the threshold $\delta_e$ during inference, a newly generated edge will be added to $\mathcal{E}$.
In our implementation, node generation selects agent roles from a predefined role pool, and edge generation predicts directed communication links between generated nodes conditioned on the latent prior and the query representation. In practice, $T$ is bounded by the candidate role pool, and unselected roles are omitted from the generated graph. This autoregressive design enables \texttt{TopoPrior} to produce sparse, query-adaptive initial collaboration graphs for downstream refinement.

\noindent\textbf{Conditional Prior.}
The conditional prior $p'_\theta(z\mid q)$ provides a query-specific latent prior for topology generation:
\begin{equation}
p'_\theta(z\mid q)=\mathcal{N}\!\left(z;f_{\mathrm{prior}}(\mathbf{h}_q),I\right),
\end{equation}
where $f_{\mathrm{prior}}$ is a trainable MLP. Since reference graphs are unavailable at inference time, this conditional prior plays an important role in mapping unseen queries to useful regions of the topology latent space. The resulting topology-prior learning objective is
\begin{multline}
\mathcal{L}_{\mathrm{prior}}
=
-
\mathbb{E}_{z\sim q_\phi(z\mid \mathbb{G},q)}
\big[\log p_\theta(\mathbb{G}\mid z,q)\big] \\
+
\mathrm{KL}\!\left(q_\phi(z\mid \mathbb{G},q)\,\|\,p'_\theta(z\mid q)\right).
\end{multline}

\subsection{Query-Conditioned Latent Adaptation}
While transferable topology priors can improve initialization efficiency, effective collaboration still requires sensitivity to query-relevant structural cues that may vary across domains. To this end, we introduce a domain-adversarial discriminator on the latent variable $z$ to reduce unnecessary domain discrepancy during training. The discriminator is defined as
\begin{gather}
f_\Psi(z)=f_d(\mathbf{W}_d z+\mathbf{b}_d),\\
\mathcal{L}_{\mathrm{adapt}}
=
-
\mathbb{E}_{(\mathbb{G},q,y^k)\sim\mathcal{D}}
\left[y^{k\top}\log f_\Psi(z)\right],
\end{gather}
where $f_d$ is the softmax function and a gradient reversal layer (GRL)~\cite{DBLP:conf/icml/GaninL15} is applied during training. Combined with query-conditioned prior learning and graph reconstruction, this adversarial objective reduces cross-domain discrepancy in the latent space, while the conditional prior and reconstruction objective help preserve query-dependent and task-relevant structural information for topology generation.

\begin{algorithm}[!t]
\newcommand{\comm}[1]{\textcolor{gray!80}{\textit{#1}}}
\caption{\texttt{TopoPrior} Training Process}
\label{training_algo}
\footnotesize
\begin{algorithmic}[1]
\REQUIRE Multi-domain training set $\mathcal{D} = \bigcup_{k=1}^{K} \mathcal{D}^{k}$, hyperparameter $\alpha$ and $\beta$, learning rate $\eta$, batch size $B$, number of epochs $E$
\ENSURE Trained \texttt{TopoPrior} parameters $\theta$, $\phi$, $\Psi$

\STATE Initialize variational encoder $q_\phi(z|\mathbb{G}, q)$ with GCN layers
\STATE Initialize conditional generator $p_\theta(\mathbb{G}|z, q)$, prior network $p'_\theta(z|q)$, and domain discriminator $f_\Psi(z)$

\FOR{epoch $= 1$ to $E$}
    \FOR{batch $(q_i, \mathbb{G}_i, d_i) \sim \mathcal{D}$}
        \STATE \comm{Step 1: Encode graph and query}
        \STATE Compute node representations $\mathbf{h}_v$ via GCN layers
        \STATE Aggregate graph representation $\mathbf{h}_{\mathbb{G}} = \sum_{v} \mathbf{h}_v$
        \STATE Encode query $q_i$ into $\mathbf{h}_q$
        \STATE Fuse query and graph features: $\mathbf{h}_{\text{task}} = \mathrm{MLP}_{\text{task}}([\mathbf{h}_q \parallel \mathbf{h}_{\mathbb{G}}])$
        \STATE Sample latent variable $z \sim q_\phi(z|\mathbb{G}, q)$ via reparameterization

        \STATE \comm{Step 2: Reconstruct graph}
        \STATE Generate nodes and edges autoregressively via $p_\theta(\mathbb{G}|z, q)$

        \STATE \comm{Step 3: Compute losses}
        \STATE $\mathcal{L}_{\text{recon}} = -\mathbb{E}_{z \sim q_\phi}[\log p_\theta(\mathbb{G}|z, q)]$
        \STATE $\mathcal{L}_{\text{KL}} = \mathrm{KL}(q_\phi(z|\mathbb{G}, q) \| p'_\theta(z|q))$
        \STATE $\mathcal{L}_{\text{adapt}} = -\mathbb{E}[y \cdot \log f_\Psi(z)]$

        \STATE \comm{Step 4: Update parameters}
        \STATE $\mathcal{L}_{\text{prior}} = \mathcal{L}_{\text{recon}} + \mathcal{L}_{\text{KL}}$
        \STATE $\mathcal{L}_{\text{total}} = \mathcal{L}_{\text{prior}} + \alpha \mathcal{L}_{\mathrm{task}} + \beta \mathcal{L}_{\text{adapt}}$
        \STATE Update $\theta, \phi, \Psi$ via $\nabla_{\theta, \phi, \Psi} \mathcal{L}_{\text{total}}$
    \ENDFOR
\ENDFOR
\end{algorithmic}
\end{algorithm}

\subsection{Training and Inference}
The training objective of \texttt{TopoPrior} combines topology-prior learning and latent-space alignment:
\begin{equation}
\mathcal{L}_{\mathrm{total}}
=
\mathcal{L}_{\mathrm{prior}}
+
\alpha\,\mathcal{L}_{\mathrm{task}}
+
\beta\,\mathcal{L}_{\mathrm{adapt}},
\end{equation}
where $\alpha,\beta\geq 0$ are hyperparameters, and $\mathcal{L}_{\mathrm{task}}$ denotes the loss of the underlying task.

During training, \texttt{TopoPrior} learns transferable topology priors from reference graphs constructed on the training domains. During inference, given a new query $q$, we first sample $z\sim p'_\theta(z\mid q)$ and then generate an initial collaboration graph $\hat{\mathbb{G}}\sim p_\theta(\mathbb{G}\mid z,q)$. The generated graph is subsequently passed to a downstream topology-evolution backbone, which refines it according to its original task-specific optimization process. In this way, \texttt{TopoPrior} complements existing topology-evolution methods by providing lightweight and transferable graph initialization rather than replacing their task-specific search mechanisms. The training procedure for \texttt{TopoPrior} is summarized in Algorithm~\ref{training_algo}.

\subsection{Analytical Perspective}
\label{sec:theory}
We next provide an analytical perspective on the two design principles underlying \texttt{TopoPrior}. The goal of this subsection is not to establish end-to-end guarantees for LLM-based multi-agent topology evolution, but rather to formalize the intuition behind latent alignment and topology initialization using standard analytical tools. Since both arguments are adapted from classical generalization and convergence analyses, we present concise proof sketches specialized to our setting. Detailed proofs are deferred to Appendix~\ref{proof_theory}.

\subsubsection{Cross-Domain Transfer via Latent Alignment}
\label{subsec:multidomain}
Our framework uses adversarial latent alignment to reduce avoidable domain discrepancy across source domains, which may improve transfer to unseen domains. Let $P_1,\dots,P_K$ denote the marginal latent distributions of the $K$ source domains induced by the encoder, and let $P_t$ denote the marginal latent distribution of a target domain. For any hypothesis $h$ in the class $\mathcal{H}$, define the expected error on domain $d$ as
\begin{equation}
\epsilon_d(h)=\mathbb{E}_{z\sim P_d}[\ell_h(z)],
\end{equation}
where $\ell_h$ is a bounded loss function.

We measure domain discrepancy in the latent space by the $\mathcal{H}\Delta\mathcal{H}$-divergence:
\begin{multline}
d_{\mathcal{H}\Delta\mathcal{H}}(P,Q)
=
2\sup_{h,h'\in\mathcal{H}}
\big|\Pr_{z\sim P}[h(z)\neq h'(z)] \\
-\Pr_{z\sim Q}[h(z)\neq h'(z)]\big|.
\end{multline}

By adapting standard multi-source domain adaptation analysis~\cite{DBLP:conf/icml/0002CZG19,DBLP:conf/uai/SiciliaAAH22} to the encoder-induced latent space, we obtain the following bound.

\begin{theorem}[Multi-Source Domain Adaptation Bound in Latent Space]
\label{thm:multi_domain}
Let $\alpha\in\Delta^{K-1}$ be convex weights over source domains, and let $P_\alpha=\sum_{k=1}^{K}\alpha_k P_k$. Then, for any $h\in\mathcal{H}$,
\begin{equation}
\epsilon_t(h)
\leq
\sum_{k=1}^{K}\alpha_k\epsilon_k(h)
+
\frac{1}{2}d_{\mathcal{H}\Delta\mathcal{H}}(P_\alpha,P_t)
+
\lambda^*_{\alpha,t},
\end{equation}
where
\begin{equation}
\lambda^*_{\alpha,t}
=
\min_{h\in\mathcal{H}}
\left(
\sum_{k=1}^{K}\alpha_k\epsilon_k(h)+\epsilon_t(h)
\right).
\end{equation}
\end{theorem}

The result follows from applying standard multi-source domain adaptation arguments to the encoder-induced latent distributions and replacing input-space discrepancy with the corresponding $\mathcal{H}\Delta\mathcal{H}$-divergence in latent space~\cite{DBLP:conf/icml/0002CZG19,DBLP:conf/uai/SiciliaAAH22}. Theorem~\ref{thm:multi_domain} therefore provides a principled perspective on why reducing latent-space domain discrepancy may be beneficial for cross-domain transfer in \texttt{TopoPrior}. We emphasize that this result should be interpreted as motivation for latent-space regularization rather than as a task-specific guarantee for the full downstream system.

\subsubsection{Topology Initialization as Search Acceleration}
\label{subsec:acceleration}
Our topology initializer is designed to provide stronger starting collaboration graphs, thereby reducing the number of refinement rounds required by downstream topology evolution. Let
\begin{equation}
J_\lambda(\mathbb{G};q,y)=\mathrm{Perf}(\mathbb{G};q,y)-\lambda C(\mathbb{G};q),
\end{equation}
denote a utility function that combines task performance and communication cost, where $\lambda\geq 0$ controls the trade-off between the two terms. Define
\begin{equation}
U_t=\mathbb{E}[J_\lambda(\mathbb{G}_t;q,y)],
\end{equation}
as the expected utility at evolution step $t$, starting from an initial graph $\mathbb{G}_0$.

To analyze the role of initialization, we adopt an abstract contraction-style assumption commonly used in convergence analysis of iterative optimization procedures~\cite{DBLP:books/sp/Nesterov04}. Although downstream topology evolution in our setting is discrete and LLM-mediated, this assumption is used only as a simplified analytical model for studying the effect of initialization quality.

\begin{assumption}[Linear Convergence]
\label{asm:contraction}
There exists $\eta\in(0,1]$ such that, for all $t$,
\begin{equation}
U^*-U_{t+1}\leq(1-\eta)(U^*-U_t),
\end{equation}
where $U^*$ is the optimal achievable utility.
\end{assumption}

Under Assumption~\ref{asm:contraction}, we obtain the following round-complexity bound.

\begin{theorem}[Rounds to $\epsilon$-Suboptimality]
\label{thm:rounds}
Under Assumption~\ref{asm:contraction}, to ensure $U^*-U_T\leq\epsilon$,
it suffices that
\begin{equation}
T\geq
\frac{\log((U^*-U_0)/\epsilon)}{\log(1/(1-\eta))},
\end{equation}
\end{theorem}

The bound is obtained by recursively unrolling Assumption~\ref{asm:contraction}, which gives $U^*-U_T\leq(1-\eta)^T(U^*-U_0)$, and then solving for $T$.

\begin{corollary}[Better Initialization Reduces Rounds]
\label{cor:rounds_reduction}
If the prior-based initialization satisfies $U_0^{\mathrm{prior}}>U_0^{\mathrm{scratch}}$, then for any $\epsilon$,
\begin{equation}
T_{\mathrm{prior}}(\epsilon)-T_{\mathrm{scratch}}(\epsilon)
\leq
\frac{
\log\!\left(
\frac{U^*-U_0^{\mathrm{prior}}}{U^*-U_0^{\mathrm{scratch}}}
\right)
}{
\log(1/(1-\eta))
}
<0.
\end{equation}
\end{corollary}

The corollary follows directly by comparing the bounds induced by $U_0^{\mathrm{prior}}$ and $U_0^{\mathrm{scratch}}$ in Theorem~\ref{thm:rounds}. It formalizes the intuition that better initialization can reduce the search cost of downstream topology evolution under the stated assumption. This interpretation is consistent with the lower token usage and fewer communication rounds observed in our experiments, although the result should be viewed as an analytical perspective rather than a literal guarantee for multi-agent LLM systems.

\begin{table}[!t]
\centering
\caption{Test-set statistics under the domain partitions used in our experiments. For training, we follow the official train/validation splits provided by MMLU and C-Eval.}
\scriptsize
\setlength{\tabcolsep}{4pt}
\renewcommand{\arraystretch}{1.0}
\begin{tabular}{p{0.32\columnwidth}c p{0.32\columnwidth}c}
\toprule
\textbf{Domain} & \textbf{Test Size} & \textbf{Domain} & \textbf{Test Size} \\
\midrule
\multicolumn{4}{c}{\textbf{MMLU}} \\
\midrule
Natural Sciences & 6696 & Ethics and Morality & 558 \\
Engineering and Technology & 1116 & Business and Management & 1453 \\
Social Sciences & 2072 & Humanities and History & 1674 \\
\makecell[l]{Law, Government, \\ and Public Affairs} & 1825 & & \\
\midrule
\multicolumn{4}{c}{\textbf{C-Eval}} \\
\midrule
Natural Sciences & 2312 & \makecell[l]{Vocational Qualifications \\ and Professional Examinations} & 1581 \\
Engineering and Technology & 2424 & Medicine and Life Sciences & 1227 \\
Social Sciences and Humanities & 4618 & & \\
\midrule
\multicolumn{4}{c}{\textbf{C-Eval Hard}} \\
\midrule
Mathematics & 835 & Chemistry & 581 \\
Physics & 529 & & \\
\bottomrule
\end{tabular}
\label{dataset_size}
\vspace{-1em}
\end{table}

\section{Experiments}
\label{all_exp}
In this section, we evaluate \texttt{TopoPrior} on multi-domain reasoning benchmarks and compare it with representative baselines, with a particular focus on topology-evolution methods. We report both downstream task performance and efficiency-related metrics, including online inference-time token usage and communication rounds. We further conduct ablation, sensitivity, and generalization analyses to examine the effectiveness of transferable topology prior learning under different settings.

\subsection{Datasets}
\label{app_datasets}
We evaluate \texttt{TopoPrior} on two widely used LLM benchmarks, MMLU~\cite{DBLP:conf/iclr/HendrycksBBZMSS21} and C-Eval~\cite{DBLP:conf/nips/HuangBZZZSLLZLF23}, under multi-domain settings. To construct domains for topology-prior learning, we group fine-grained benchmark subcategories into a smaller number of major disciplinary categories according to semantic relatedness and the original benchmark taxonomy. The resulting test-set statistics are summarized in Table~\ref{dataset_size}.

\noindent\textbf{MMLU.}
MMLU contains 57 subject-level tasks. We organize them into seven major domains according to their disciplinary themes:
\textit{(1) Natural Sciences,}
\textit{(2) Engineering and Technology,}
\textit{(3) Social Sciences,}
\textit{(4) Humanities and History,}
\textit{(5) Law, Government, and Public Affairs,}
\textit{(6) Ethics and Morality,} and
\textit{(7) Business and Management.}
Together, these domains cover quantitative reasoning, technical problem solving, legal and policy analysis, ethical judgment, and knowledge-intensive social and historical understanding.

\noindent\textbf{C-Eval.}
C-Eval contains 52 subject-level tasks designed to assess the general understanding capability of Chinese LLMs. We group these tasks into five major domains:
\textit{(1) Natural Sciences,}
\textit{(2) Engineering and Technology,}
\textit{(3) Social Sciences and Humanities,}
\textit{(4) Vocational Qualifications and Professional Examinations,} and
\textit{(5) Medicine and Life Sciences.}
These categories cover both academic disciplines and professional qualification scenarios, providing a diverse testbed for multi-domain collaboration.

\noindent\textbf{C-Eval Hard.}
To further evaluate robustness on challenging scientific reasoning tasks, we construct \emph{C-Eval Hard} by extracting three natural-science subdomains from C-Eval:
\textit{(1) Mathematics,}
\textit{(2) Chemistry,} and
\textit{(3) Physics.}
This benchmark focuses on difficult quantitative reasoning and serves as an additional evaluation setting for cross-domain transfer under domain variation.

\subsection{Baselines}
\label{baselines_app}
\noindent\textbf{Training-free methods.}
We consider vanilla prompt engineering (PE) with Llama3-8B-Instruct~\cite{touvron2023}, Qwen2.5-72B-Instruct~\cite{DBLP:journals/corr/abs-2412-15115}, and DeepSeek-V3-671B-Instruct~\cite{deepseekai2025} as representative single-agent backbones. We further include Chain-of-Thought (CoT)~\cite{DBLP:conf/nips/Wei0SBIXCLZ22} and Retrieval-Augmented Generation (RAG)~\cite{DBLP:conf/nips/LewisPPPKGKLYR020} as stronger prompting-based baselines.

\noindent\textbf{Training-intensive methods.}
We compare with representative multi-domain adaptation methods, including MoDULA~\cite{DBLP:conf/emnlp/0011LDCGRZJJZCY24}, MoDE~\cite{DBLP:journals/corr/abs-2410-10181}, and DES-MoE~\cite{DBLP:journals/corr/abs-2509-16882}. These methods improve domain specialization through parameter-efficient fine-tuning or mixture-of-experts architectures.

\noindent\textbf{Training-light multi-agent methods.}
We evaluate \texttt{TopoPrior} on top of several topology-evolution backbones, including G-Designer~\cite{DBLP:conf/icml/ZhangYSWYF00C25}, AgentPrune~\cite{DBLP:conf/iclr/ZhangYLYWWCY025}, ARG-Designer~\cite{DBLP:conf/aaai/LiLWZP26}, and AgentDropout~\cite{DBLP:conf/acl/WangW00Z0025}. These methods dynamically optimize collaboration topologies through graph learning, pruning, autoregressive generation, or agent/edge dropout, respectively. Our method is complementary to these approaches, as it provides transferable graph initialization while leaving their downstream topology-refinement procedures unchanged.

\begin{table*}[!t]
\centering
\caption{Agent roles used in \texttt{TopoPrior}, aligned with the domain partitions of MMLU and C-Eval. Roles are selected dynamically according to the query content and task context.}
\tiny
\setlength{\tabcolsep}{4pt}
\renewcommand{\arraystretch}{1.1}
\begin{tabular}{>{\centering\arraybackslash}p{0.15\textwidth}
                >{\centering\arraybackslash}p{0.16\textwidth}
                >{\arraybackslash}p{0.41\textwidth}
                >{\centering\arraybackslash}p{0.18\textwidth}}
\toprule
\textbf{Role Name} & \textbf{Domain} & \textbf{Role Description} & \textbf{Sub-domain} \\
\midrule
Natural Science Expert & Natural Sciences &
Provides knowledge in physics, chemistry, biology, and medicine, and handles formal reasoning with domain-specific terminology.
& Math, Physics, Chemistry, Biology \\

Engineering Specialist & Engineering \& Technology &
Solves queries in computer science, algorithms, security, and systems design using structured technical reasoning.
& Computer Science, ML, Security \\

Social Scientist & Social Sciences &
Analyzes economic, psychological, sociological, and political concepts, and performs causal and qualitative reasoning.
& Economics, Psychology, Political Science \\

Humanities Scholar & Humanities \& History &
Interprets historical events, philosophical arguments, cultural contexts, and ethical frameworks.
& History, Philosophy, World Religions \\

Legal Analyst & Law, Government, and Public Affairs &
Interprets legal texts, statutes, treaties, and policy documents, and performs normative and juridical reasoning.
& Jurisprudence, International Law \\

Ethics Consultant & Ethics \& Morality &
Evaluates moral dilemmas, ethical scenarios, and value-based judgments, especially in subjective contexts.
& Moral Controversies, Ethical Scenarios \\

Business Strategist & Business \& Management &
Analyzes business ethics, accounting, marketing, and management strategies by combining normative and financial reasoning.
& Business Ethics, Accounting, Marketing \\

Mathematical Expert & Mathematics &
Solves mathematical queries, including discrete mathematics, probability, statistics, and algebra.
& Discrete Math, Probability, Statistics \\

Chemistry Specialist & Chemistry &
Answers chemistry questions at the high-school and college levels, and explains chemical reactions, properties, and theories.
& General Chemistry, Organic Chemistry \\

Physics Specialist & Physics &
Handles physics questions at the high-school and college levels, and applies principles from mechanics, electromagnetism, and related areas.
& Mechanics, Electromagnetism \\

Medical Life Scientist & Medical &
Provides knowledge in clinical medicine, veterinary science, agronomy, and plant sciences.
& Clinical Medicine, Veterinary Science \\

Vocational Examiner & Vocational Qualifications &
Answers certification-oriented questions in finance, taxation, civil service, and tourism.
& CPA, Tax Agent, Civil Service \\

General Coordinator & Multi-domain &
Orchestrates collaboration, aggregates outputs, and manages communication flow among specialized agents.
& All domains (coordination role) \\
\bottomrule
\end{tabular}
\label{agent_roles}
\vspace{-1em}
\end{table*}

\begin{table*}[!t]
\centering
\caption{Results on MMLU across domains with different LLM backbones. ``GD'', ``AP'', ``ARG'', and ``AD'' denote G-Designer, AgentPrune, ARG-Designer, and AgentDropout, respectively. Best results are in bold and second-best results are underlined.}
\setlength{\tabcolsep}{2pt}
\begin{tabular}{ccccccccc}
\toprule
\multicolumn{1}{c}{\multirow{1}{*}{\textbf{Domain}$ \quad \rightarrow$}} & 
\multirow{2}{*}{\makecell{\textbf{Natural} \\ \textbf{Sciences}}} &
\multirow{2}{*}{\makecell{\textbf{Engineering} \\ \textbf{Technology}}} & 
\multirow{2}{*}{\makecell{\textbf{Social} \\ \textbf{Sciences}}} & 
\multirow{2}{*}{\makecell{\textbf{Humanities} \\ \textbf{History}}} & 
\multirow{2}{*}{\makecell{\textbf{Law, Government} \\ \textbf{Public Affairs}}} & 
\multirow{2}{*}{\makecell{\textbf{Ethics} \\ \textbf{Morality}}} & 
\multirow{2}{*}{\makecell{\textbf{Business} \\ \textbf{Management}}} & 
\multirow{2}{*}{\textbf{Average}} \\
\multicolumn{1}{c}{\multirow{1}{*}{\textbf{Models}$  \quad \downarrow$}}  & & & & & & & & \\
\midrule
\addlinespace[0.3pt]
\midrule
\multicolumn{9}{c}{Base model: Llama3-8B-Instruct} \\
\midrule
\addlinespace[0.3pt]
\midrule
PE & 57.84 & 58.89 & 56.61 & 60.05 & 49.84 & 45.72 & 59.40 & 55.05 \\
CoT & 58.97$_{\textcolor{orange}{(\uparrow 1.13)}}$ & 61.51$_{\textcolor{orange}{(\uparrow 2.62)}}$ & 59.02$_{\textcolor{orange}{(\uparrow 2.41)}}$ & 62.48$_{\textcolor{orange}{(\uparrow 2.43)}}$ & 50.65$_{\textcolor{orange}{(\uparrow 0.81)}}$ & 46.76$_{\textcolor{orange}{(\uparrow 1.04)}}$ & 60.68$_{\textcolor{orange}{(\uparrow 1.28)}}$ & 57.15\\
RAG & 61.65$_{\textcolor{orange}{(\uparrow 3.81)}}$ & 62.39$_{\textcolor{orange}{(\uparrow 3.50)}}$ & 62.95$_{\textcolor{orange}{(\uparrow 6.34)}}$ & 65.37$_{\textcolor{orange}{(\uparrow 5.32)}}$ & 58.03$_{\textcolor{orange}{(\uparrow 8.19)}}$ & 54.50$_{\textcolor{orange}{(\uparrow 8.78)}}$ & 62.83$_{\textcolor{orange}{(\uparrow 3.43)}}$ & 61.10  \\ \hdashline
MoDULA & 56.89$_{\textcolor{mygreen}{(\downarrow 0.95)}}$ & 58.72$_{\textcolor{mygreen}{(\downarrow 0.17)}}$ & 57.95$_{\textcolor{orange}{(\uparrow 1.34)}}$ & 61.22$_{\textcolor{orange}{(\uparrow 1.17)}}$ & 51.60$_{\textcolor{orange}{(\uparrow 1.76)}}$ & 48.42$_{\textcolor{orange}{(\uparrow 2.70)}}$ & 60.53$_{\textcolor{orange}{(\uparrow 1.13)}}$ & 56.47 \\
MoDE & 59.72$_{\textcolor{orange}{(\uparrow 1.88)}}$ & 63.05$_{\textcolor{orange}{(\uparrow 4.16)}}$ & 62.45$_{\textcolor{orange}{(\uparrow 5.84)}}$ & 64.69$_{\textcolor{orange}{(\uparrow 4.64)}}$ & 56.79$_{\textcolor{orange}{(\uparrow 6.95)}}$ & 54.11$_{\textcolor{orange}{(\uparrow 8.39)}}$ & 61.06$_{\textcolor{orange}{(\uparrow 1.66)}}$ & 60.12 \\
DES-MoE & 58.88$_{\textcolor{orange}{(\uparrow 1.04)}}$ & 62.34$_{\textcolor{orange}{(\uparrow 3.45)}}$ & 60.70$_{\textcolor{orange}{(\uparrow 4.09)}}$ & 62.63$_{\textcolor{orange}{(\uparrow 2.58)}}$ & 54.84$_{\textcolor{orange}{(\uparrow 5.00)}}$ & 52.98$_{\textcolor{orange}{(\uparrow 7.26)}}$ & 61.33$_{\textcolor{orange}{(\uparrow 1.93)}}$ & 59.10 \\ \hdashline
G-Designer & 64.82$_{\textcolor{orange}{(\uparrow 6.98)}}$ & 66.57$_{\textcolor{orange}{(\uparrow 7.68)}}$ & 63.91$_{\textcolor{orange}{(\uparrow 7.30)}}$ & 67.70$_{\textcolor{orange}{(\uparrow 7.65)}}$ & 61.58$_{\textcolor{orange}{(\uparrow 11.74)}}$ & 59.13$_{\textcolor{orange}{(\uparrow 13.41)}}$ & 65.44$_{\textcolor{orange}{(\uparrow 6.04)}}$ & 64.16  \\
AgentPrune & 63.99$_{\textcolor{orange}{(\uparrow 6.15)}}$ & 66.85$_{\textcolor{orange}{(\uparrow 7.96)}}$ & 64.73$_{\textcolor{orange}{(\uparrow 8.12)}}$ & 66.98$_{\textcolor{orange}{(\uparrow 6.93)}}$ & 63.07$_{\textcolor{orange}{(\uparrow 13.23)}}$ & 58.51$_{\textcolor{orange}{(\uparrow 12.79)}}$ & 64.12$_{\textcolor{orange}{(\uparrow 4.72)}}$ & 64.04 \\
ARG-Designer & 65.48$_{\textcolor{orange}{(\uparrow 7.64)}}$ & 68.01$_{\textcolor{orange}{(\uparrow 9.12)}}$ & 64.95$_{\textcolor{orange}{(\uparrow 8.34)}}$ & 69.54$_{\textcolor{orange}{(\uparrow 9.49)}}$ & 63.86$_{\textcolor{orange}{(\uparrow 14.02)}}$ & 61.32$_{\textcolor{orange}{(\uparrow 15.60)}}$ & 65.87$_{\textcolor{orange}{(\uparrow 6.47)}}$ & 65.58 \\
AgentDropout & 64.30$_{\textcolor{orange}{(\uparrow 6.46)}}$ & 69.72$_{\textcolor{orange}{(\uparrow 10.83)}}$ & 63.14$_{\textcolor{orange}{(\uparrow 6.53)}}$ & 68.86$_{\textcolor{orange}{(\uparrow 8.81)}}$ & 65.47$_{\textcolor{orange}{(\uparrow 15.63)}}$ & 62.05$_{\textcolor{orange}{(\uparrow 16.33)}}$ & 66.53$_{\textcolor{orange}{(\uparrow 7.13)}}$ & 65.72 \\ \hdashline
\rowcolor{mylightgray} {\texttt{TopoPrior}+GD} & \underline{66.32}$_{\textcolor{orange}{(\uparrow 8.48)}}$ & 67.80$_{\textcolor{orange}{(\uparrow 8.91)}}$ & 65.42$_{\textcolor{orange}{(\uparrow 8.81)}}$ & \underline{70.41}$_{\textcolor{orange}{(\uparrow 10.36)}}$ & 63.96$_{\textcolor{orange}{(\uparrow 14.12)}}$ & 60.85$_{\textcolor{orange}{(\uparrow 15.13)}}$ & 67.18$_{\textcolor{orange}{(\uparrow 7.78)}}$ & 65.99  \\
\rowcolor{mylightgray} {\texttt{TopoPrior}+AP} & 65.34$_{\textcolor{orange}{(\uparrow 7.50)}}$ & 68.75$_{\textcolor{orange}{(\uparrow 9.86)}}$ & 65.19$_{\textcolor{orange}{(\uparrow 8.58)}}$ & 69.16$_{\textcolor{orange}{(\uparrow 9.11)}}$ & 63.93$_{\textcolor{orange}{(\uparrow 14.09)}}$ & 61.80$_{\textcolor{orange}{(\uparrow 16.08)}}$ & 67.62$_{\textcolor{orange}{(\uparrow 8.22)}}$ & 65.97 \\
\rowcolor{mylightgray} {\texttt{TopoPrior}+ARG} & \textbf{68.91}$_{\textcolor{orange}{(\uparrow 11.07)}}$ & \textbf{71.57}$_{\textcolor{orange}{(\uparrow 12.68)}}$ & \textbf{68.13}$_{\textcolor{orange}{(\uparrow 11.52)}}$ & \textbf{72.74}$_{\textcolor{orange}{(\uparrow 12.69)}}$ & \underline{66.15}$_{\textcolor{orange}{(\uparrow 16.31)}}$ & \textbf{64.08}$_{\textcolor{orange}{(\uparrow 18.36)}}$ & \underline{68.14}$_{\textcolor{orange}{(\uparrow 8.74)}}$ & \textbf{68.53} \\
\rowcolor{mylightgray} {\texttt{TopoPrior}+AD} & 65.70$_{\textcolor{orange}{(\uparrow 7.86)}}$ & \underline{69.95}$_{\textcolor{orange}{(\uparrow 11.06)}}$ & \underline{65.47}$_{\textcolor{orange}{(\uparrow 8.86)}}$ & 70.36$_{\textcolor{orange}{(\uparrow 10.31)}}$ & \textbf{66.88}$_{\textcolor{orange}{(\uparrow 17.04)}}$ & \underline{63.75}$_{\textcolor{orange}{(\uparrow 18.03)}}$ & \textbf{69.05}$_{\textcolor{orange}{(\uparrow 9.65)}}$ & \underline{67.31} \\ \midrule
\addlinespace[0.3pt]
\midrule
\multicolumn{9}{c}{Base model: DeepSeek-V3-671B-Instruct} \\
\midrule
\addlinespace[0.3pt]
\midrule

PE & 84.96 & 85.23 & 84.78 & 86.41 & 79.63 & 78.92 & 85.12 & 84.47 \\
CoT & 84.60$_{\textcolor{mygreen}{(\downarrow 0.36)}}$ & 84.85$_{\textcolor{mygreen}{(\downarrow 0.38)}}$ & 85.32$_{\textcolor{orange}{(\uparrow 0.54)}}$ & 86.87$_{\textcolor{orange}{(\uparrow 0.46)}}$ & 80.19$_{\textcolor{orange}{(\uparrow 0.56)}}$ & 79.52$_{\textcolor{orange}{(\uparrow 0.60)}}$ & 85.57$_{\textcolor{orange}{(\uparrow 0.45)}}$ & 84.90 \\
RAG & 85.92$_{\textcolor{orange}{(\uparrow 0.96)}}$ & 86.17$_{\textcolor{orange}{(\uparrow 0.94)}}$ & 85.84$_{\textcolor{orange}{(\uparrow 1.06)}}$ & 87.63$_{\textcolor{orange}{(\uparrow 1.22)}}$ & 82.45$_{\textcolor{orange}{(\uparrow 2.82)}}$ & 81.78$_{\textcolor{orange}{(\uparrow 2.86)}}$ & 86.05$_{\textcolor{orange}{(\uparrow 0.93)}}$ & 85.38  \\ \hdashline
G-Designer & 89.23$_{\textcolor{orange}{(\uparrow 4.27)}}$ & 89.67$_{\textcolor{orange}{(\uparrow 4.44)}}$ & 88.91$_{\textcolor{orange}{(\uparrow 4.13)}}$ & 91.08$_{\textcolor{orange}{(\uparrow 4.67)}}$ & 84.12$_{\textcolor{orange}{(\uparrow 4.49)}}$ & 83.46$_{\textcolor{orange}{(\uparrow 4.54)}}$ & 89.35$_{\textcolor{orange}{(\uparrow 4.23)}}$ & 88.52  \\
AgentPrune & 89.84$_{\textcolor{orange}{(\uparrow 4.88)}}$ & 90.03$_{\textcolor{orange}{(\uparrow 4.80)}}$ & 89.52$_{\textcolor{orange}{(\uparrow 4.74)}}$ & 91.47$_{\textcolor{orange}{(\uparrow 5.06)}}$ & 84.73$_{\textcolor{orange}{(\uparrow 5.10)}}$ & 83.98$_{\textcolor{orange}{(\uparrow 5.06)}}$ & 89.76$_{\textcolor{orange}{(\uparrow 4.64)}}$ & 89.19 \\
ARG-Designer & 91.15$_{\textcolor{orange}{(\uparrow 6.19)}}$ & 91.80$_{\textcolor{orange}{(\uparrow 6.57)}}$ & 91.07$_{\textcolor{orange}{(\uparrow 6.29)}}$ & 93.69$_{\textcolor{orange}{(\uparrow 7.28)}}$ & 86.28$_{\textcolor{orange}{(\uparrow 6.65)}}$ & 85.52$_{\textcolor{orange}{(\uparrow 6.60)}}$ & 91.48$_{\textcolor{orange}{(\uparrow 6.36)}}$ & 90.14 \\
AgentDropout & 89.97$_{\textcolor{orange}{(\uparrow 5.01)}}$ & 90.50$_{\textcolor{orange}{(\uparrow 5.27)}}$ & 90.11$_{\textcolor{orange}{(\uparrow 5.33)}}$ & 91.93$_{\textcolor{orange}{(\uparrow 5.52)}}$ & 85.29$_{\textcolor{orange}{(\uparrow 5.66)}}$ & 84.89$_{\textcolor{orange}{(\uparrow 5.97)}}$ & 90.25$_{\textcolor{orange}{(\uparrow 5.13)}}$ & 88.99 \\ \hdashline
\rowcolor{mylightgray} {\texttt{TopoPrior}+GD} & 90.12$_{\textcolor{orange}{(\uparrow 5.16)}}$ & 90.46$_{\textcolor{orange}{(\uparrow 5.23)}}$ & 89.82$_{\textcolor{orange}{(\uparrow 5.04)}}$ & 92.37$_{\textcolor{orange}{(\uparrow 5.96)}}$ & 85.43$_{\textcolor{orange}{(\uparrow 5.80)}}$ & 84.71$_{\textcolor{orange}{(\uparrow 5.79)}}$ & 90.28$_{\textcolor{orange}{(\uparrow 5.16)}}$ & 89.89  \\
\rowcolor{mylightgray} {\texttt{TopoPrior}+AP} & 90.75$_{\textcolor{orange}{(\uparrow 5.79)}}$ & 90.90$_{\textcolor{orange}{(\uparrow 5.67)}}$ & 90.41$_{\textcolor{orange}{(\uparrow 5.63)}}$ & 92.89$_{\textcolor{orange}{(\uparrow 6.48)}}$ & 86.01$_{\textcolor{orange}{(\uparrow 6.38)}}$ & 85.31$_{\textcolor{orange}{(\uparrow 6.39)}}$ & 90.65$_{\textcolor{orange}{(\uparrow 5.53)}}$ & 90.41 \\
\rowcolor{mylightgray} {\texttt{TopoPrior}+ARG} & \textbf{92.87}$_{\textcolor{orange}{(\uparrow 7.91)}}$ & \textbf{93.14}$_{\textcolor{orange}{(\uparrow 7.91)}}$ & \textbf{92.55}$_{\textcolor{orange}{(\uparrow 7.77)}}$ & \textbf{95.12}$_{\textcolor{orange}{(\uparrow 8.71)}}$ & \textbf{87.83}$_{\textcolor{orange}{(\uparrow 8.20)}}$ & \textbf{87.08}$_{\textcolor{orange}{(\uparrow 8.16)}}$ & \textbf{92.79}$_{\textcolor{orange}{(\uparrow 7.67)}}$ & \textbf{92.03} \\
\rowcolor{mylightgray} {\texttt{TopoPrior}+AD} & \underline{91.78}$_{\textcolor{orange}{(\uparrow 6.82)}}$ & \underline{92.09}$_{\textcolor{orange}{(\uparrow 6.86)}}$ & \underline{91.33}$_{\textcolor{orange}{(\uparrow 6.55)}}$ & \underline{93.95}$_{\textcolor{orange}{(\uparrow 7.54)}}$ & \underline{86.95}$_{\textcolor{orange}{(\uparrow 7.32)}}$ & \underline{86.51}$_{\textcolor{orange}{(\uparrow 7.59)}}$ & \underline{91.74}$_{\textcolor{orange}{(\uparrow 6.62)}}$ & \underline{90.62} \\

\midrule
\addlinespace[0.3pt]
\midrule
\multicolumn{9}{c}{Base model: Qwen2.5-72B-Instruct} \\
\midrule
\addlinespace[0.3pt]
\midrule

PE & 84.17 & 84.44 & 83.86 & 85.53 & 79.84 & 79.19 & 83.92 & 83.26 \\
CoT & 85.02$_{\textcolor{orange}{(\uparrow 0.85)}}$ & 84.40$_{\textcolor{mygreen}{(\downarrow 0.04)}}$ & 84.79$_{\textcolor{orange}{(\uparrow 0.93)}}$ & 86.71$_{\textcolor{orange}{(\uparrow 1.18)}}$ & 81.17$_{\textcolor{orange}{(\uparrow 1.33)}}$ & 80.42$_{\textcolor{orange}{(\uparrow 1.23)}}$ & 85.07$_{\textcolor{orange}{(\uparrow 1.15)}}$ & 84.41\\
RAG & 85.71$_{\textcolor{orange}{(\uparrow 1.54)}}$ & 86.02$_{\textcolor{orange}{(\uparrow 1.58)}}$ & 85.48$_{\textcolor{orange}{(\uparrow 1.62)}}$ & 87.43$_{\textcolor{orange}{(\uparrow 1.90)}}$ & 82.33$_{\textcolor{orange}{(\uparrow 2.49)}}$ & 81.58$_{\textcolor{orange}{(\uparrow 2.39)}}$ & 85.73$_{\textcolor{orange}{(\uparrow 1.81)}}$ & 85.56  \\ \hdashline
G-Designer & 87.42$_{\textcolor{orange}{(\uparrow 3.25)}}$ & 87.55$_{\textcolor{orange}{(\uparrow 3.11)}}$ & 87.20$_{\textcolor{orange}{(\uparrow 3.34)}}$ & 88.97$_{\textcolor{orange}{(\uparrow 3.44)}}$ & 83.19$_{\textcolor{orange}{(\uparrow 3.35)}}$ & 82.58$_{\textcolor{orange}{(\uparrow 3.39)}}$ & 87.43$_{\textcolor{orange}{(\uparrow 3.51)}}$ & 86.33  \\
AgentPrune & 86.67$_{\textcolor{orange}{(\uparrow 2.50)}}$ & 86.98$_{\textcolor{orange}{(\uparrow 2.54)}}$ & 86.44$_{\textcolor{orange}{(\uparrow 2.58)}}$ & 88.42$_{\textcolor{orange}{(\uparrow 2.89)}}$ & 82.56$_{\textcolor{orange}{(\uparrow 2.72)}}$ & 81.81$_{\textcolor{orange}{(\uparrow 2.62)}}$ & 86.69$_{\textcolor{orange}{(\uparrow 2.77)}}$ & 85.71 \\
ARG-Designer & 88.50$_{\textcolor{orange}{(\uparrow 4.33)}}$ & 88.87$_{\textcolor{orange}{(\uparrow 4.43)}}$ & 88.41$_{\textcolor{orange}{(\uparrow 4.55)}}$ & 90.40$_{\textcolor{orange}{(\uparrow 4.87)}}$ & 84.49$_{\textcolor{orange}{(\uparrow 4.65)}}$ & 83.55$_{\textcolor{orange}{(\uparrow 4.36)}}$ & 88.69$_{\textcolor{orange}{(\uparrow 4.77)}}$ & 87.56 \\
AgentDropout & 87.63$_{\textcolor{orange}{(\uparrow 3.46)}}$ & 87.79$_{\textcolor{orange}{(\uparrow 3.35)}}$ & 87.40$_{\textcolor{orange}{(\uparrow 3.54)}}$ & 89.32$_{\textcolor{orange}{(\uparrow 3.79)}}$ & 83.59$_{\textcolor{orange}{(\uparrow 3.75)}}$ & 82.70$_{\textcolor{orange}{(\uparrow 3.51)}}$ & 87.74$_{\textcolor{orange}{(\uparrow 3.82)}}$ & 86.60 \\ \hdashline
\rowcolor{mylightgray} {\texttt{TopoPrior}+GD} & 88.31$_{\textcolor{orange}{(\uparrow 4.14)}}$ & 88.44$_{\textcolor{orange}{(\uparrow 4.00)}}$ & 88.13$_{\textcolor{orange}{(\uparrow 4.27)}}$ & 90.01$_{\textcolor{orange}{(\uparrow 4.48)}}$ & 84.30$_{\textcolor{orange}{(\uparrow 4.46)}}$ & 83.29$_{\textcolor{orange}{(\uparrow 4.10)}}$ & 88.32$_{\textcolor{orange}{(\uparrow 4.40)}}$ & 87.26  \\
\rowcolor{mylightgray} {\texttt{TopoPrior}+AP} & 87.74$_{\textcolor{orange}{(\uparrow 3.57)}}$ & 87.61$_{\textcolor{orange}{(\uparrow 3.27)}}$ & 87.49$_{\textcolor{orange}{(\uparrow 3.63)}}$ & 89.33$_{\textcolor{orange}{(\uparrow 3.80)}}$ & 83.58$_{\textcolor{orange}{(\uparrow 3.74)}}$ & 82.73$_{\textcolor{orange}{(\uparrow 3.54)}}$ & 87.64$_{\textcolor{orange}{(\uparrow 3.72)}}$ & 86.59 \\
\rowcolor{mylightgray} {\texttt{TopoPrior}+ARG} & \textbf{90.87}$_{\textcolor{orange}{(\uparrow 6.70)}}$ & \textbf{91.18}$_{\textcolor{orange}{(\uparrow 6.74)}}$ & \textbf{90.64}$_{\textcolor{orange}{(\uparrow 6.78)}}$ & \textbf{92.62}$_{\textcolor{orange}{(\uparrow 7.09)}}$ & \textbf{86.76}$_{\textcolor{orange}{(\uparrow 6.92)}}$ & \textbf{86.01}$_{\textcolor{orange}{(\uparrow 6.82)}}$ & \textbf{90.89}$_{\textcolor{orange}{(\uparrow 6.97)}}$ & \textbf{89.56} \\
\rowcolor{mylightgray} {\texttt{TopoPrior}+AD} & \underline{89.48}$_{\textcolor{orange}{(\uparrow 5.31)}}$ & \underline{89.62}$_{\textcolor{orange}{(\uparrow 5.18)}}$ & \underline{89.35}$_{\textcolor{orange}{(\uparrow 5.49)}}$ & \underline{91.26}$_{\textcolor{orange}{(\uparrow 5.73)}}$ & \underline{85.45}$_{\textcolor{orange}{(\uparrow 5.61)}}$ & \underline{84.75}$_{\textcolor{orange}{(\uparrow 5.56)}}$ & \underline{89.89}$_{\textcolor{orange}{(\uparrow 5.97)}}$ & \underline{88.54} \\

\bottomrule
\end{tabular}
\label{main_res_mmlu}
\vspace{-.5em}
\end{table*}

\subsection{Implementation Details}
\label{implementation_settings}
We implement \texttt{TopoPrior} in PyTorch and conduct experiments on two NVIDIA A800 GPUs. Unless otherwise specified, \texttt{TopoPrior} is trained separately on each benchmark under its corresponding multi-domain partition.

\noindent\textbf{Backbones and encoders.}
Agent roles are instantiated using LLM backbones, including the locally deployed Llama3-8B-Instruct~\cite{touvron2023} and the online models Qwen2.5-72B-Instruct~\cite{DBLP:journals/corr/abs-2412-15115} and DeepSeek-V3-671B-Instruct~\cite{deepseekai2025}. The query representation $\mathbf{h}_q$ and the initial agent-role representation $\mathbf{h}_v^{(0)}$ are encoded using a frozen BERT encoder~\cite{DBLP:conf/naacl/DevlinCLT19}. Unless otherwise stated, gold domain labels are used only for benchmark partitioning and analysis and are not provided to \texttt{TopoPrior} at inference time.

\noindent\textbf{\texttt{TopoPrior} architecture.}
The variational encoder is implemented as a two-layer GCN~\cite{DBLP:conf/iclr/KipfW17} with hidden size $d_{\mathbf{h}_v}=256$, latent dimension $d_z=128$, and ReLU activation~\cite{DBLP:journals/corr/abs-1803-08375}. Linear projections are implemented with two-layer MLPs. For node and edge history encoding, we use a two-layer GRU~\cite{DBLP:journals/corr/ChungGCB14}. The latent-space discriminator is implemented as a two-layer MLP with a softmax output layer. The edge-generation threshold is set to $\delta_e=0.5$, and the gradient reversal layer uses a coefficient of $-0.1$.

\noindent\textbf{Training protocol.}
We optimize the model with Adam using a learning rate of $2\times 10^{-4}$ and a batch size of 32. The alignment coefficients are set to $\alpha=0.5$, $\beta=0.5$, and the topology-prior generator is trained for five epochs. Reference collaboration graphs are constructed offline on the training split using AgentDropout~\cite{DBLP:conf/acl/WangW00Z0025}. During evaluation, the generated initial graph is passed to the corresponding downstream topology-evolution backbone for task-specific refinement. For fair comparison, integrating \texttt{TopoPrior} changes only the initialization stage of each backbone while keeping its downstream refinement procedure and search budget unchanged, unless otherwise specified.

\noindent\textbf{Role pool and graph construction.}
The conditional generator $p_\theta(\mathbb{G}\mid z,q)$ autoregressively selects agent nodes from an extensible role pool. Table~\ref{agent_roles} summarizes the agent roles associated with the major domains in MMLU and C-Eval. For natural-science tasks, we include both a generic \textit{Natural Science Expert} and specialized roles (e.g., \textit{Mathematical Expert}, \textit{Chemistry Specialist}, and \textit{Physics Specialist}) to better support challenging subdomains. For efficient offline supervision construction, we use AgentDropout~\cite{DBLP:conf/acl/WangW00Z0025} to obtain reference communication topologies, as it is substantially faster than more iterative topology-construction methods on large training sets. Each agent node $v_{i,j}^{k}=\{\texttt{LLM}, \texttt{Role}, \texttt{State}, \texttt{Plugins}\}$ encodes the underlying language model, assigned role, current state, and any attached tools or plugins~\cite{DBLP:conf/iclr/ZhangYLYWWCY025}.

\noindent\textbf{Evaluation protocol.}
Unless otherwise specified, we report classification accuracy (Acc.) on domain-specific test splits. For efficiency analysis, token consumption is measured as the total number of LLM input and output tokens incurred during online inference, including inter-agent communication and final answer generation, but excluding the offline cost of reference-graph construction and topology-prior training. We report these metrics as online inference efficiency and discuss offline supervision and training cost separately when relevant. For ARG-Designer~\cite{DBLP:conf/aaai/LiLWZP26}, which generates graph topologies autoregressively from scratch, we incorporate the latent representation from \texttt{TopoPrior} into its initial query encoding (i.e., $\mathbf{f}_{\mathcal{Q}}$ in its Eq.~(6)) as the initialization signal.

\begin{table*}[!t]
\centering
\caption{Results on C-Eval and C-Eval Hard across domains with LLM backbones . ``GD'', ``AP'', ``ARG'', and ``AD'' denote G-Designer, AgentPrune, ARG-Designer, and AgentDropout. Best results are in bold and second-best results are underlined.}
\setlength{\tabcolsep}{2pt}
\begin{tabular}{cccccccccc}
\toprule
\multicolumn{1}{c}{\multirow{1}{*}{\textbf{Domain}$ \quad \rightarrow$}} & 
\multirow{2}{*}{ \makecell{\textbf{Natural} \\ \textbf{Sciences}}} &
\multirow{2}{*}{\makecell{\textbf{Engineering} \\ \textbf{Technology}}} & 
\multirow{2}{*}{\makecell{\textbf{Social Sciences} \\ \textbf{Humanities}}} & 
\multirow{2}{*}{\makecell{\textbf{Vocational Qualif.} \\ \textbf{Prof. Exam.}}} & 
\multirow{2}{*}{\makecell{\textbf{Medicine} \\ \textbf{Life Sci.}}} & 
\multicolumn{1}{c}{\cellcolor[gray]{0.9}} &
\multicolumn{1}{c}{\cellcolor[gray]{0.9}\textbf{C-Eval Hard}} &
\multicolumn{1}{c}{\cellcolor[gray]{0.9}} & \multirow{2}{*}{\textbf{Average}} \\
\multicolumn{1}{c}{\multirow{1}{*}{\textbf{Models}$  \quad \downarrow$}}  & & & & & & \multirow{1}{*}{\makecell{\textbf{Math}}} & \multirow{1}{*}{\makecell{\textbf{Chemistry}}} & \multirow{1}{*}{\textbf{Physics}}   \\ \midrule
\addlinespace[0.3pt]
\midrule
\multicolumn{9}{c}{Base model: Llama3-8B-Instruct} \\
\midrule
\addlinespace[0.3pt]
\midrule
PE & 54.73 & 53.64 & 55.02 & 48.69 & 46.25 & 41.83 & 44.06 & 40.12 & 48.04 \\
CoT & 55.85$_{\textcolor{orange}{(\uparrow 1.12)}}$ & 54.47$_{\textcolor{orange}{(\uparrow 0.83)}}$ & 55.48$_{\textcolor{orange}{(\uparrow 0.46)}}$ & 50.36$_{\textcolor{orange}{(\uparrow 1.67)}}$ & 48.07$_{\textcolor{orange}{(\uparrow 1.82)}}$ & 42.99$_{\textcolor{orange}{(\uparrow 1.16)}}$ & 45.80$_{\textcolor{orange}{(\uparrow 1.74)}}$ & 41.84$_{\textcolor{orange}{(\uparrow 1.72)}}$ & 49.36 \\
RAG & 57.45$_{\textcolor{orange}{(\uparrow 2.72)}}$ & 56.59$_{\textcolor{orange}{(\uparrow 2.95)}}$ & 56.98$_{\textcolor{orange}{(\uparrow 1.96)}}$ & 51.75$_{\textcolor{orange}{(\uparrow 3.06)}}$ & 50.58$_{\textcolor{orange}{(\uparrow 4.33)}}$ & 44.83$_{\textcolor{orange}{(\uparrow 3.00)}}$ & 47.06$_{\textcolor{orange}{(\uparrow 3.00)}}$ & 44.12$_{\textcolor{orange}{(\uparrow 4.00)}}$ & 51.17 \\ \hdashline
MoDULA & 54.37$_{\textcolor{mygreen}{(\downarrow 0.36)}}$ & 52.95$_{\textcolor{mygreen}{(\downarrow 0.69)}}$ & 55.84$_{\textcolor{orange}{(\uparrow 0.82)}}$ & 49.72$_{\textcolor{orange}{(\uparrow 1.03)}}$ & 47.55$_{\textcolor{orange}{(\uparrow 1.30)}}$ & 41.80$_{\textcolor{mygreen}{(\downarrow 0.03)}}$ & 45.62$_{\textcolor{orange}{(\uparrow 1.56)}}$ & 41.23$_{\textcolor{orange}{(\uparrow 1.11)}}$ & 48.64 \\
MoDE & 54.61$_{\textcolor{orange}{(\uparrow 0.12)}}$ & 53.97$_{\textcolor{orange}{(\uparrow 0.33)}}$ & 56.14$_{\textcolor{orange}{(\uparrow 1.12)}}$ & 50.80$_{\textcolor{orange}{(\uparrow 2.11)}}$ & 49.25$_{\textcolor{orange}{(\uparrow 3.00)}}$ & 42.77$_{\textcolor{orange}{(\uparrow 0.94)}}$ & 44.02$_{\textcolor{mygreen}{(\downarrow 0.04)}}$ & 42.52$_{\textcolor{orange}{(\uparrow 2.40)}}$ & 49.26 \\
DES-MoE & 55.15$_{\textcolor{orange}{(\uparrow 0.42)}}$ & 53.42$_{\textcolor{mygreen}{(\downarrow 0.22)}}$ & 56.50$_{\textcolor{orange}{(\uparrow 1.48)}}$ & 51.31$_{\textcolor{orange}{(\uparrow 2.62)}}$ & 48.84$_{\textcolor{orange}{(\uparrow 2.59)}}$ & 43.96$_{\textcolor{orange}{(\uparrow 2.13)}}$ & 43.58$_{\textcolor{mygreen}{(\downarrow 0.48)}}$ & 41.96$_{\textcolor{orange}{(\uparrow 1.84)}}$ & 49.34 \\ \hdashline
G-Designer & \underline{59.60}$_{\textcolor{orange}{(\uparrow 4.87)}}$ & 57.86$_{\textcolor{orange}{(\uparrow 4.22)}}$ & 57.15$_{\textcolor{orange}{(\uparrow 2.13)}}$ & 54.08$_{\textcolor{orange}{(\uparrow 5.39)}}$ & 52.63$_{\textcolor{orange}{(\uparrow 6.38)}}$ & 45.91$_{\textcolor{orange}{(\uparrow 4.08)}}$ & 50.11$_{\textcolor{orange}{(\uparrow 6.05)}}$ & 46.32$_{\textcolor{orange}{(\uparrow 6.20)}}$ & 52.96  \\
AgentPrune & 59.48$_{\textcolor{orange}{(\uparrow 4.75)}}$ & 58.16$_{\textcolor{orange}{(\uparrow 4.52)}}$ & 56.93$_{\textcolor{orange}{(\uparrow 1.91)}}$ & 54.50$_{\textcolor{orange}{(\uparrow 5.81)}}$ & 52.09$_{\textcolor{orange}{(\uparrow 5.84)}}$ & 45.24$_{\textcolor{orange}{(\uparrow 3.41)}}$ & 50.85$_{\textcolor{orange}{(\uparrow 6.79)}}$ & 45.88$_{\textcolor{orange}{(\uparrow 5.76)}}$ & 52.89  \\
ARG-Designer & 60.82$_{\textcolor{orange}{(\uparrow 6.09)}}$ & 59.52$_{\textcolor{orange}{(\uparrow 5.88)}}$ & 57.11$_{\textcolor{orange}{(\uparrow 2.09)}}$ & 56.27$_{\textcolor{orange}{(\uparrow 7.58)}}$ & 53.85$_{\textcolor{orange}{(\uparrow 7.60)}}$ & 46.73$_{\textcolor{orange}{(\uparrow 4.90)}}$ & 52.59$_{\textcolor{orange}{(\uparrow 8.53)}}$ & 47.14$_{\textcolor{orange}{(\uparrow 7.02)}}$ & 54.25 \\
AgentDropout & 60.40$_{\textcolor{orange}{(\uparrow 5.67)}}$ & 58.86$_{\textcolor{orange}{(\uparrow 5.22)}}$ & 57.10$_{\textcolor{orange}{(\uparrow 2.08)}}$ & 55.79$_{\textcolor{orange}{(\uparrow 7.10)}}$ & 53.54$_{\textcolor{orange}{(\uparrow 7.29)}}$ & 45.98$_{\textcolor{orange}{(\uparrow 4.15)}}$ & 52.10$_{\textcolor{orange}{(\uparrow 8.04)}}$ & 46.96$_{\textcolor{orange}{(\uparrow 6.84)}}$ & 53.84 \\ \hdashline
\rowcolor{mylightgray} {\texttt{TopoPrior}+GD} & 61.83$_{\textcolor{orange}{(\uparrow 7.10)}}$ & 58.72$_{\textcolor{orange}{(\uparrow 5.08)}}$ & 59.05$_{\textcolor{orange}{(\uparrow 4.03)}}$ & 56.51$_{\textcolor{orange}{(\uparrow 7.82)}}$ & 54.30$_{\textcolor{orange}{(\uparrow 8.05)}}$ & 46.95$_{\textcolor{orange}{(\uparrow 5.12)}}$ & 52.63$_{\textcolor{orange}{(\uparrow 8.57)}}$ & 48.41$_{\textcolor{orange}{(\uparrow 8.29)}}$ & 54.80 \\
\rowcolor{mylightgray} {\texttt{TopoPrior}+AP} & 61.16$_{\textcolor{orange}{(\uparrow 6.43)}}$ & 60.35$_{\textcolor{orange}{(\uparrow 6.71)}}$ & 58.84$_{\textcolor{orange}{(\uparrow 3.82)}}$ & 56.07$_{\textcolor{orange}{(\uparrow 7.38)}}$ & 53.98$_{\textcolor{orange}{(\uparrow 7.73)}}$ & 47.72$_{\textcolor{orange}{(\uparrow 5.89)}}$ & 52.93$_{\textcolor{orange}{(\uparrow 8.87)}}$ & 48.22$_{\textcolor{orange}{(\uparrow 8.10)}}$ & 54.91 \\
\rowcolor{mylightgray} {\texttt{TopoPrior}+ARG} & \textbf{63.54}$_{\textcolor{orange}{(\uparrow 8.81)}}$ & \textbf{62.60}$_{\textcolor{orange}{(\uparrow 8.96)}}$ & \textbf{60.21}$_{\textcolor{orange}{(\uparrow 5.19)}}$ & \textbf{58.15}$_{\textcolor{orange}{(\uparrow 9.46)}}$ & \textbf{55.82}$_{\textcolor{orange}{(\uparrow 9.57)}}$ & \textbf{49.33}$_{\textcolor{orange}{(\uparrow 7.50)}}$ & \textbf{55.76}$_{\textcolor{orange}{(\uparrow 11.70)}}$ & \textbf{50.89}$_{\textcolor{orange}{(\uparrow 10.77)}}$ & \textbf{57.04} \\
\rowcolor{mylightgray} {\texttt{TopoPrior}+AD} & \underline{62.80}$_{\textcolor{orange}{(\uparrow 8.07)}}$ & \underline{61.42}$_{\textcolor{orange}{(\uparrow 7.78)}}$ & \underline{58.96}$_{\textcolor{orange}{(\uparrow 3.94)}}$ & \underline{57.40}$_{\textcolor{orange}{(\uparrow 8.71)}}$ & \underline{55.36}$_{\textcolor{orange}{(\uparrow 9.11)}}$ & \underline{48.44}$_{\textcolor{orange}{(\uparrow 6.61)}}$ & \underline{53.86}$_{\textcolor{orange}{(\uparrow 9.80)}}$ & \underline{49.87}$_{\textcolor{orange}{(\uparrow 9.75)}}$ & \underline{56.01}  \\ \midrule
\addlinespace[0.3pt]
\midrule
\multicolumn{9}{c}{Base model: DeepSeek-V3-671B-Instruct} \\
\midrule
\addlinespace[0.3pt]
\midrule

PE & 86.88 & 84.75 & 87.94 & 79.76 & 77.39 & 71.21 & 74.48 & 69.64 & 79.01 \\
CoT & 86.63$_{\textcolor{mygreen}{(\downarrow 0.25)}}$ & 85.62$_{\textcolor{orange}{(\uparrow 0.87)}}$ & 88.64$_{\textcolor{orange}{(\uparrow 0.70)}}$ & 81.75$_{\textcolor{orange}{(\uparrow 1.99)}}$ & 78.35$_{\textcolor{orange}{(\uparrow 0.96)}}$ & 72.78$_{\textcolor{orange}{(\uparrow 1.57)}}$ & 75.62$_{\textcolor{orange}{(\uparrow 1.14)}}$ & 70.81$_{\textcolor{orange}{(\uparrow 1.17)}}$ & 80.03 \\
RAG & 88.82$_{\textcolor{orange}{(\uparrow 1.94)}}$ & 86.88$_{\textcolor{orange}{(\uparrow 2.13)}}$ & 89.70$_{\textcolor{orange}{(\uparrow 1.76)}}$ & 81.92$_{\textcolor{orange}{(\uparrow 2.16)}}$ & 79.54$_{\textcolor{orange}{(\uparrow 2.15)}}$ & 73.39$_{\textcolor{orange}{(\uparrow 2.18)}}$ & 76.85$_{\textcolor{orange}{(\uparrow 2.37)}}$ & 72.03$_{\textcolor{orange}{(\uparrow 2.39)}}$ & 81.14 \\ \hdashline
G-Designer & 90.35$_{\textcolor{orange}{(\uparrow 3.47)}}$ & 88.52$_{\textcolor{orange}{(\uparrow 3.77)}}$ & 91.40$_{\textcolor{orange}{(\uparrow 3.46)}}$ & 83.43$_{\textcolor{orange}{(\uparrow 3.67)}}$ & 80.82$_{\textcolor{orange}{(\uparrow 3.43)}}$ & 75.28$_{\textcolor{orange}{(\uparrow 4.07)}}$ & 78.93$_{\textcolor{orange}{(\uparrow 4.45)}}$ & 74.14$_{\textcolor{orange}{(\uparrow 4.50)}}$ & 82.86 \\
AgentPrune & 90.62$_{\textcolor{orange}{(\uparrow 3.74)}}$ & 88.84$_{\textcolor{orange}{(\uparrow 4.09)}}$ & 90.98$_{\textcolor{orange}{(\uparrow 3.04)}}$ & 84.11$_{\textcolor{orange}{(\uparrow 4.35)}}$ & 81.08$_{\textcolor{orange}{(\uparrow 3.69)}}$ & 74.45$_{\textcolor{orange}{(\uparrow 3.24)}}$ & 81.15$_{\textcolor{orange}{(\uparrow 6.67)}}$ & 75.32$_{\textcolor{orange}{(\uparrow 5.68)}}$ & 83.32 \\
ARG-Designer & 91.49$_{\textcolor{orange}{(\uparrow 4.61)}}$ & 89.87$_{\textcolor{orange}{(\uparrow 5.12)}}$ & 91.90$_{\textcolor{orange}{(\uparrow 3.96)}}$ & {85.83}$_{\textcolor{orange}{(\uparrow 6.07)}}$ & 82.76$_{\textcolor{orange}{(\uparrow 5.37)}}$ & 76.50$_{\textcolor{orange}{(\uparrow 5.29)}}$ & 81.97$_{\textcolor{orange}{(\uparrow 7.49)}}$ & 76.06$_{\textcolor{orange}{(\uparrow 6.42)}}$ & 84.42 \\
AgentDropout & 90.83$_{\textcolor{orange}{(\uparrow 3.95)}}$ & 88.15$_{\textcolor{orange}{(\uparrow 3.40)}}$ & 91.22$_{\textcolor{orange}{(\uparrow 3.28)}}$ & 83.64$_{\textcolor{orange}{(\uparrow 3.88)}}$ & 81.66$_{\textcolor{orange}{(\uparrow 4.27)}}$ & 76.17$_{\textcolor{orange}{(\uparrow 4.96)}}$ & 80.94$_{\textcolor{orange}{(\uparrow 6.46)}}$ & 75.80$_{\textcolor{orange}{(\uparrow 6.16)}}$ & 83.55 \\ \hdashline
\rowcolor{mylightgray} {\texttt{TopoPrior}+GD} & \underline{91.80}$_{\textcolor{orange}{(\uparrow 4.92)}}$ & 89.67$_{\textcolor{orange}{(\uparrow 4.92)}}$ & \underline{92.52}$_{\textcolor{orange}{(\uparrow 4.58)}}$ & 85.33$_{\textcolor{orange}{(\uparrow 5.57)}}$ & 82.95$_{\textcolor{orange}{(\uparrow 5.56)}}$ & 76.84$_{\textcolor{orange}{(\uparrow 5.63)}}$ & 80.50$_{\textcolor{orange}{(\uparrow 6.02)}}$ & 76.69$_{\textcolor{orange}{(\uparrow 7.05)}}$ & 84.54 \\
\rowcolor{mylightgray} {\texttt{TopoPrior}+AP} & 91.19$_{\textcolor{orange}{(\uparrow 4.31)}}$ & \underline{90.51}$_{\textcolor{orange}{(\uparrow 5.76)}}$ & 92.27$_{\textcolor{orange}{(\uparrow 4.33)}}$ & 85.30$_{\textcolor{orange}{(\uparrow 5.54)}}$ & 82.78$_{\textcolor{orange}{(\uparrow 5.39)}}$ & 77.03$_{\textcolor{orange}{(\uparrow 5.82)}}$ & \underline{82.41}$_{\textcolor{orange}{(\uparrow 7.93)}}$ & \underline{76.98}$_{\textcolor{orange}{(\uparrow 7.34)}}$ & 84.81 \\
\rowcolor{mylightgray} {\texttt{TopoPrior}+ARG} & \textbf{93.02}$_{\textcolor{orange}{(\uparrow 6.14)}}$ & \textbf{92.46}$_{\textcolor{orange}{(\uparrow 7.71)}}$ & \textbf{93.08}$_{\textcolor{orange}{(\uparrow 5.14)}}$ & \textbf{86.68}$_{\textcolor{orange}{(\uparrow 6.92)}}$ & \textbf{84.83}$_{\textcolor{orange}{(\uparrow 7.44)}}$ & \textbf{78.31}$_{\textcolor{orange}{(\uparrow 7.10)}}$ & \textbf{83.94}$_{\textcolor{orange}{(\uparrow 9.46)}}$ & \textbf{78.06}$_{\textcolor{orange}{(\uparrow 8.42)}}$ & 86.30 \\
\rowcolor{mylightgray} {\texttt{TopoPrior}+AD} & 91.75$_{\textcolor{orange}{(\uparrow 4.87)}}$ & 90.32$_{\textcolor{orange}{(\uparrow 5.57)}}$ & 91.79$_{\textcolor{orange}{(\uparrow 3.85)}}$ &  \underline{85.96}$_{\textcolor{orange}{(\uparrow 5.00)}}$ & \underline{83.03}$_{\textcolor{orange}{(\uparrow 5.64)}}$ & \underline{77.14}$_{\textcolor{orange}{(\uparrow 5.93)}}$ & 81.95$_{\textcolor{orange}{(\uparrow 7.47)}}$ & 76.40$_{\textcolor{orange}{(\uparrow 6.76)}}$ & 84.79 \\

\midrule
\addlinespace[0.3pt]
\midrule
\multicolumn{9}{c}{Base model: Qwen2.5-72B-Instruct} \\
\midrule
\addlinespace[0.3pt]
\midrule

PE & 84.42 & 82.55 & 85.29 & 78.30 & 76.57 & 70.55 & 74.02 & 68.18 & 77.49 \\
CoT & 84.53$_{\textcolor{orange}{(\uparrow 0.11)}}$ & 82.10$_{\textcolor{mygreen}{(\downarrow 0.45)}}$ & 85.27$_{\textcolor{mygreen}{(\downarrow 0.02)}}$ & 80.19$_{\textcolor{orange}{(\uparrow 1.89)}}$ & 78.24$_{\textcolor{orange}{(\uparrow 1.67)}}$ & 71.79$_{\textcolor{orange}{(\uparrow 1.24)}}$ & 75.33$_{\textcolor{orange}{(\uparrow 1.31)}}$ & 69.85$_{\textcolor{orange}{(\uparrow 1.67)}}$ & 78.41 \\
RAG & 85.40$_{\textcolor{orange}{(\uparrow 0.98)}}$ & 84.06$_{\textcolor{orange}{(\uparrow 1.51)}}$ & 86.55$_{\textcolor{orange}{(\uparrow 1.26)}}$ & 82.12$_{\textcolor{orange}{(\uparrow 3.82)}}$ & 79.69$_{\textcolor{orange}{(\uparrow 3.12)}}$ & 73.40$_{\textcolor{orange}{(\uparrow 2.85)}}$ & 76.73$_{\textcolor{orange}{(\uparrow 2.71)}}$ & 71.31$_{\textcolor{orange}{(\uparrow 3.13)}}$ & 79.91 \\ \hdashline
G-Designer & 86.58$_{\textcolor{orange}{(\uparrow 2.16)}}$ & 85.21$_{\textcolor{orange}{(\uparrow 2.66)}}$ & 88.07$_{\textcolor{orange}{(\uparrow 2.78)}}$ & 84.29$_{\textcolor{orange}{(\uparrow 5.99)}}$ & 80.28$_{\textcolor{orange}{(\uparrow 3.71)}}$ & 74.99$_{\textcolor{orange}{(\uparrow 4.44)}}$ & 78.87$_{\textcolor{orange}{(\uparrow 4.85)}}$ & 73.96$_{\textcolor{orange}{(\uparrow 5.78)}}$ & 81.53 \\
AgentPrune & 86.03$_{\textcolor{orange}{(\uparrow 1.61)}}$ & 86.15$_{\textcolor{orange}{(\uparrow 3.60)}}$ & 88.22$_{\textcolor{orange}{(\uparrow 2.93)}}$ & 85.93$_{\textcolor{orange}{(\uparrow 7.63)}}$ & 80.79$_{\textcolor{orange}{(\uparrow 4.22)}}$ & 75.47$_{\textcolor{orange}{(\uparrow 4.92)}}$ & 79.34$_{\textcolor{orange}{(\uparrow 5.32)}}$ & 74.84$_{\textcolor{orange}{(\uparrow 6.66)}}$ & 82.10 \\
ARG-Designer & 87.63$_{\textcolor{orange}{(\uparrow 3.21)}}$ & \underline{87.72}$_{\textcolor{orange}{(\uparrow 5.17)}}$ & 89.16$_{\textcolor{orange}{(\uparrow 3.87)}}$ & 86.57$_{\textcolor{orange}{(\uparrow 8.27)}}$ & 81.85$_{\textcolor{orange}{(\uparrow 5.28)}}$ & 76.74$_{\textcolor{orange}{(\uparrow 6.19)}}$ & 81.41$_{\textcolor{orange}{(\uparrow 7.39)}}$ & 76.45$_{\textcolor{orange}{(\uparrow 8.27)}}$ & 82.19 \\
AgentDropout & 86.57$_{\textcolor{orange}{(\uparrow 2.15)}}$ & 85.70$_{\textcolor{orange}{(\uparrow 3.15)}}$ & 87.96$_{\textcolor{orange}{(\uparrow 2.67)}}$ & 85.52$_{\textcolor{orange}{(\uparrow 7.22)}}$ & 80.44$_{\textcolor{orange}{(\uparrow 3.87)}}$ & 75.30$_{\textcolor{orange}{(\uparrow 4.75)}}$ & 79.39$_{\textcolor{orange}{(\uparrow 5.37)}}$ & 74.68$_{\textcolor{orange}{(\uparrow 6.50)}}$ & 81.95 \\ \hdashline
\rowcolor{mylightgray} {\texttt{TopoPrior}+GD} & 87.65$_{\textcolor{orange}{(\uparrow 3.23)}}$ & 87.10$_{\textcolor{orange}{(\uparrow 4.55)}}$ & 88.74$_{\textcolor{orange}{(\uparrow 3.45)}}$ & 86.65$_{\textcolor{orange}{(\uparrow 8.35)}}$ & 82.11$_{\textcolor{orange}{(\uparrow 5.54)}}$ & 77.16$_{\textcolor{orange}{(\uparrow 6.61)}}$ & \underline{81.52}$_{\textcolor{orange}{(\uparrow 7.50)}}$ & 76.51$_{\textcolor{orange}{(\uparrow 8.33)}}$ & 82.18 \\
\rowcolor{mylightgray} {\texttt{TopoPrior}+AP} & \underline{88.67}$_{\textcolor{orange}{(\uparrow 4.25)}}$ & 86.84$_{\textcolor{orange}{(\uparrow 4.29)}}$ & \underline{89.73}$_{\textcolor{orange}{(\uparrow 4.44)}}$ & \underline{86.92}$_{\textcolor{orange}{(\uparrow 8.62)}}$ & 81.90$_{\textcolor{orange}{(\uparrow 5.33)}}$ & \underline{77.65}$_{\textcolor{orange}{(\uparrow 7.10)}}$ & 80.44$_{\textcolor{orange}{(\uparrow 6.42)}}$ & 76.69$_{\textcolor{orange}{(\uparrow 8.51)}}$ & 82.36 \\
\rowcolor{mylightgray} {\texttt{TopoPrior}+ARG} & \textbf{90.16}$_{\textcolor{orange}{(\uparrow 5.74)}}$ & \textbf{89.25}$_{\textcolor{orange}{(\uparrow 6.70)}}$ & \textbf{91.21}$_{\textcolor{orange}{(\uparrow 5.92)}}$ & \textbf{88.07}$_{\textcolor{orange}{(\uparrow 9.77)}}$ & \textbf{83.32}$_{\textcolor{orange}{(\uparrow 6.75)}}$ & \textbf{79.09}$_{\textcolor{orange}{(\uparrow 8.54)}}$ & \textbf{83.50}$_{\textcolor{orange}{(\uparrow 9.48)}}$ & \textbf{79.91}$_{\textcolor{orange}{(\uparrow 11.73)}}$ & 85.56 \\
\rowcolor{mylightgray} {\texttt{TopoPrior}+AD} & 88.59$_{\textcolor{orange}{(\uparrow 4.17)}}$ & 87.68$_{\textcolor{orange}{(\uparrow 5.13)}}$ & 88.64$_{\textcolor{orange}{(\uparrow 3.35)}}$ & 86.53$_{\textcolor{orange}{(\uparrow 8.23)}}$ & \underline{82.81}$_{\textcolor{orange}{(\uparrow 6.24)}}$ & 77.55$_{\textcolor{orange}{(\uparrow 7.00)}}$ & 81.36$_{\textcolor{orange}{(\uparrow 7.34)}}$ & \underline{77.14}$_{\textcolor{orange}{(\uparrow 8.96)}}$ & 83.79 \\
\bottomrule
\end{tabular}
\label{main_res_ceval}
\vspace{-.5em}
\end{table*}

\subsection{Main Results}
Table~\ref{main_res_mmlu} and Table~\ref{main_res_ceval} summarize the main results on MMLU and C-Eval under three LLM backbones. Overall, dynamic topology-evolution methods consistently outperform single-agent baselines in most evaluated settings, and equipping them with \texttt{TopoPrior} further improves performance in most cases. These results suggest that transferable topology priors can provide effective initialization for downstream multi-agent collaboration in multi-domain reasoning.

Compared with training-free methods, RAG is the strongest single-agent baseline in most settings, especially on knowledge-intensive domains such as \emph{Law, Government, and Public Affairs}, \emph{Ethics and Morality}, and \emph{Social Sciences}. Training-intensive methods exhibit more mixed behavior across heterogeneous domains and are generally weaker than dynamic multi-agent approaches in our evaluation. By contrast, topology-evolution methods provide the strongest baseline results overall, which supports the value of role specialization and structured communication for multi-domain reasoning.

Built on top of these topology-evolution backbones, \texttt{TopoPrior} yields consistent gains across benchmarks and model scales. For example, on MMLU with Llama3-8B-Instruct, \texttt{TopoPrior}+ARG improves the average accuracy from \textbf{65.58} to \textbf{68.53} (\textbf{+2.95} points), and \texttt{TopoPrior}+AD improves AgentDropout from \textbf{65.72} to \textbf{67.31} (\textbf{+1.59} points). On C-Eval with the same backbone, \texttt{TopoPrior}+ARG improves ARG-Designer from \textbf{54.25} to \textbf{57.04} (\textbf{+2.79} points), while \texttt{TopoPrior}+AD improves AgentDropout from \textbf{53.84} to \textbf{56.01} (\textbf{+2.17} points). Similar trends are observed on larger backbones, including DeepSeek-V3-671B-Instruct and Qwen2.5-72B-Instruct.

The improvements are particularly visible in reasoning-intensive and knowledge-intensive categories. On MMLU with Llama3-8B-Instruct, \texttt{TopoPrior}+ARG achieves the best performance in \emph{Ethics and Morality} (\textbf{64.08}) and ranks among the strongest methods in several other domains, while \texttt{TopoPrior}+AD performs best on \emph{Law, Government, and Public Affairs} (\textbf{66.88}) and \emph{Business and Management} (\textbf{69.05}). On C-Eval and C-Eval Hard, the gains are also evident on challenging scientific reasoning tasks. For instance, under Llama3-8B-Instruct, \texttt{TopoPrior}+ARG reaches \textbf{49.33}, \textbf{55.76}, and \textbf{50.89} on Mathematics, Chemistry, and Physics, respectively, outperforming the corresponding backbone without topology-prior initialization.

Among the evaluated backbones, ARG-Designer appears to benefit the most from \texttt{TopoPrior}. This trend is observed across both benchmarks and all three LLM backbones, suggesting that autoregressive graph construction may be particularly sensitive to initialization quality. More broadly, the fact that \texttt{TopoPrior} improves several heterogeneous backbones, rather than only the method used to construct the reference graphs, suggests that it captures reusable collaboration regularities beyond a teacher-specific search pattern.

\begin{table}[!t]
\centering
\caption{Ablation results of \texttt{TopoPrior} on three representative MMLU domains using ARG-Designer.}
\begin{tabular}{lccc}
\toprule
\multicolumn{1}{c}{\multirow{1}{*}{\textbf{Domain}$ \rightarrow$}} & \multirow{2}{*}{\makecell{\textbf{Natural}\\\textbf{Sciences}}} & \multirow{2}{*}{\makecell{\textbf{Law, Gov.}\\\textbf{Public Affairs}}} & \multirow{2}{*}{\makecell{\textbf{Ethics}\\\textbf{Morality}}}  \\
\multicolumn{1}{c}{\multirow{1}{*}{\textbf{Model}$ \downarrow$}} & & & \\
\midrule
\multicolumn{4}{c}{Base model: Llama3-8B-Instruct} \\
\midrule
\rowcolor{mylightgray} \multicolumn{1}{c}{\texttt{TopoPrior}+ARG} & \textbf{68.91} & \textbf{66.15} & \textbf{64.08} \\
\midrule
\multicolumn{1}{c}{w/o $\mathcal{L}_{\mathrm{prior}}$} & 64.71 & 62.92 & 59.81 \\
\multicolumn{1}{c}{w/o $\mathcal{L}_{\mathrm{adapt}}$} & 66.14 & 64.68 & 62.50 \\
\multicolumn{1}{c}{w/o $f_{\mathrm{prior}}$} & 65.31 & 63.70 & 60.25 \\
\midrule
\multicolumn{4}{c}{Base model: Qwen2.5-72B-Instruct} \\
\midrule
\rowcolor{mylightgray} \multicolumn{1}{c}{\texttt{TopoPrior}+ARG} & \textbf{90.87} & \textbf{86.76} & \textbf{86.01} \\
\midrule
\multicolumn{1}{c}{w/o $\mathcal{L}_{\mathrm{prior}}$} & 85.15 & 83.92 & 83.26 \\
\multicolumn{1}{c}{w/o $\mathcal{L}_{\mathrm{adapt}}$} & 87.83 & 85.57 & 85.13 \\
\multicolumn{1}{c}{w/o $f_{\mathrm{prior}}$} & 86.42 & 85.06 & 84.55 \\
\bottomrule
\end{tabular}
\label{ablation_study}
\vspace{-1em}
\end{table}

\subsection{Ablation Study}
Table~\ref{ablation_study} reports ablation results on three representative MMLU domains under two LLM backbones. Using \texttt{TopoPrior}+ARG as the full model, we examine the contributions of transferable topology-prior learning, latent adaptation, and the query-conditioned prior.
Removing $\mathcal{L}_{\mathrm{prior}}$ causes the largest performance drop (up to 4.27 points on Llama3-8B and 5.72 points on Qwen2.5-72B), indicating that topology-prior learning is the central component of \texttt{TopoPrior}. Without this objective, the model is less able to capture reusable structural regularities, leading to weaker graph initialization.
Removing $\mathcal{L}_{\mathrm{adapt}}$ also degrades performance on both backbones, although less severely than removing $\mathcal{L}_{\mathrm{prior}}$. This suggests that latent-space alignment improves the robustness of the learned topology prior across domains by reducing domain discrepancy while still preserving useful query-dependent structural cues.
Removing the query-conditioned prior $f_{\mathrm{prior}}$ further reduces performance on both backbones. This result suggests that the learned prior remains important at inference time: without it, the generator loses an informative query-specific initialization signal and becomes less effective at constructing initial collaboration graphs for downstream refinement.
Overall, the three components are complementary. The topology-prior objective contributes the largest gain, while latent adaptation and the query-conditioned prior provide additional improvements for more stable topology initialization. Their relative ordering is similar across the two backbones, suggesting that these effects are reasonably stable across model scales.

\begin{figure}[!t]
\centering
\includegraphics[width=7cm]{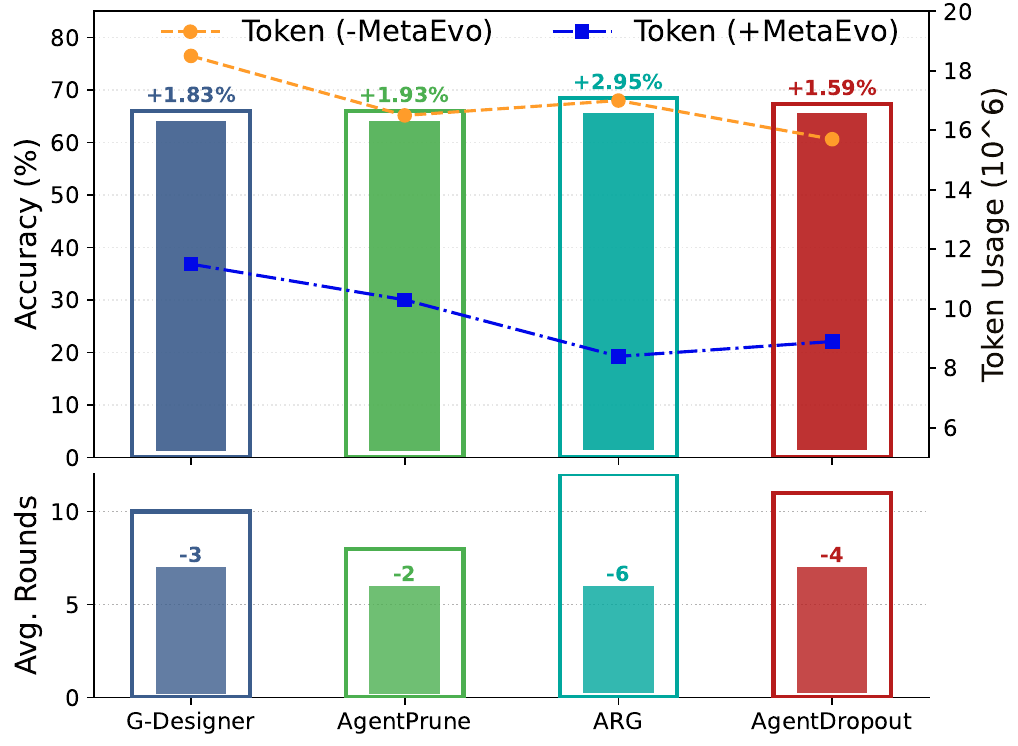}
\caption{Accuracy gain, inference-time token reduction, and communication-round reduction of \texttt{TopoPrior} when combined with different topology-evolution backbones on MMLU using Llama3-8B-Instruct.}
\label{token_cost}
\vspace{-1em}
\end{figure}

\begin{table}[!t]
\caption{Comparison of average accuracy and trainable parameters on MMLU with Llama3-8B-Instruct. Relative parameter increase ($\Delta$) is computed w.r.t. the 8B backbone.}
\centering
\begin{tabular}{lccc}
\toprule
\textbf{Method} & \textbf{Acc. (\%)} & \makecell[c]{\textbf{Trainable Parameters (M)}} & $\Delta$ (\%) \\
\midrule
MoDULA & 56.47 & 1,775 & 22.19 \\
MoDE & 60.12 & 2,376 & 29.70 \\
DES-MoE & 59.10 & 1,585 & 19.81 \\
\texttt{TopoPrior}+ARG & \textbf{68.53} & \textbf{3.3+3.8=7.1} & \textbf{0.09} \\
\bottomrule
\end{tabular}
\label{computational_para}
\vspace{-1em}
\end{table}

\subsection{Efficiency Analysis}
\label{detailed_analysis}

\paragraph{Efficiency--Performance Trade-off.}
We assess the trade-off among performance gain, online token efficiency, and communication-round reduction by equipping \texttt{TopoPrior} with different topology-evolution backbones on MMLU using Llama3-8B-Instruct. Following the protocol described in Section~\ref{implementation_settings}, token consumption is measured over online inference only. We also compare training-intensive baselines to assess adaptation quality relative to trainable parameter cost.

Figure~\ref{token_cost} shows that \texttt{TopoPrior} improves all evaluated topology-evolution methods while reducing online communication cost. The outer box labeled ``Acc.'' and the inner box labeled ``Rounds'' indicate changes in accuracy and communication rounds, respectively. Among the evaluated methods, ARG-Designer appears to benefit the most, achieving the largest average accuracy gain (\(+2.95\) points) together with the largest reduction in communication rounds (6 rounds on average). This observation is consistent with the main results and suggests that autoregressive graph construction may be particularly sensitive to initialization quality. By contrast, pruning-based methods such as AgentPrune benefit less from improved initialization, possibly because their downstream refinement process retains more redundant agents and edges.

Table~\ref{computational_para} further summarizes adaptation performance and parameter cost. \texttt{TopoPrior}+ARG achieves the best average accuracy (68.53\%) while introducing only 3.3M additional parameters in \texttt{TopoPrior} itself; the total trainable parameter count of \texttt{TopoPrior}+ARG is 7.1M, including 3.8M parameters from the original ARG-Designer. This increase is small relative to the 8B backbone. By contrast, training-intensive baselines require substantially more trainable parameters while achieving lower average accuracy on this benchmark. These results suggest that \texttt{TopoPrior} provides a favorable trade-off between adaptation quality and parameter efficiency under the evaluated setting. We note, however, that these efficiency results characterize online inference savings and parameter overhead, rather than the full amortized cost including offline reference-graph construction and prior training.

\begin{table}[!t]
\centering
\caption{Performance of \texttt{TopoPrior} under weak and strong topology supervision.}
\label{tab_weak_teacher}
\begin{tabular}{lc}
\toprule
\textbf{Setting} & \textbf{Accuracy (\%)} \\
\midrule
Weak Teacher (50\% convergence) & 61.24 \\
\texttt{TopoPrior} + Weak Teacher & 64.33 \\
Strong Teacher (full convergence) & 65.58 \\
\texttt{TopoPrior} + Strong Teacher & 68.53 \\
\bottomrule
\end{tabular}
\vspace{-1em}
\end{table}

\begin{table}[!t]
\centering
\caption{Effect of different supervision sources on model performance and inference-time token efficiency. Token reduction is measured against the ARG-Designer baseline.}
\label{tab_supervision_alternatives}
\begin{tabular}{lcc}
\toprule
\textbf{Supervision Source} & \textbf{Accuracy (\%)} & \makecell[c]{\textbf{Token} \\ \textbf{Reduction (\%)}} \\
\midrule
ARG-Designer & 65.58 & -- \\
+ \textit{Full} & 68.53 & 40.2 \\
+ \textit{Cheap-early} & 64.33 & 32.7 \\
+ \textit{Static-template} & 62.14 & 18.3 \\
+ \textit{Random} & 60.25 & 20.5 \\
\bottomrule
\end{tabular}
\vspace{-1em}
\end{table}

\paragraph{Weak Topology Supervision.}
We further examine whether \texttt{TopoPrior} remains effective under degraded and simplified supervision. Table~\ref{tab_weak_teacher} compares strong teacher graphs obtained from fully converged AgentDropout with weak teacher graphs obtained by stopping AgentDropout after 50\% of its convergence rounds. Even under weak teacher supervision, \texttt{TopoPrior} reaches 64.33\% accuracy, improving over the weak teacher itself by 3.09 points and remaining only 1.25 points below the fully supervised setting. This result suggests that \texttt{TopoPrior} can still extract useful collaboration patterns from imperfect topology supervision.

We further compare alternative supervision sources in Table~\ref{tab_supervision_alternatives}, including fully converged graphs (\textit{Full}), partially converged graphs (\textit{Cheap-early}), domain-heuristic static templates (\textit{Static-template}), and random graphs. All variants are used to train \texttt{TopoPrior}, and the resulting generator is evaluated by initializing ARG-Designer. The results show that \texttt{TopoPrior} remains effective even when supervision is simplified: partially converged graphs already provide clear improvements, while static templates still yield moderate gains. By contrast, random graphs do not produce competitive results, suggesting that \texttt{TopoPrior} learns meaningful collaboration patterns rather than merely fitting superficial graph statistics.

\begin{figure}[!t]
\centering
\includegraphics[width=6cm]{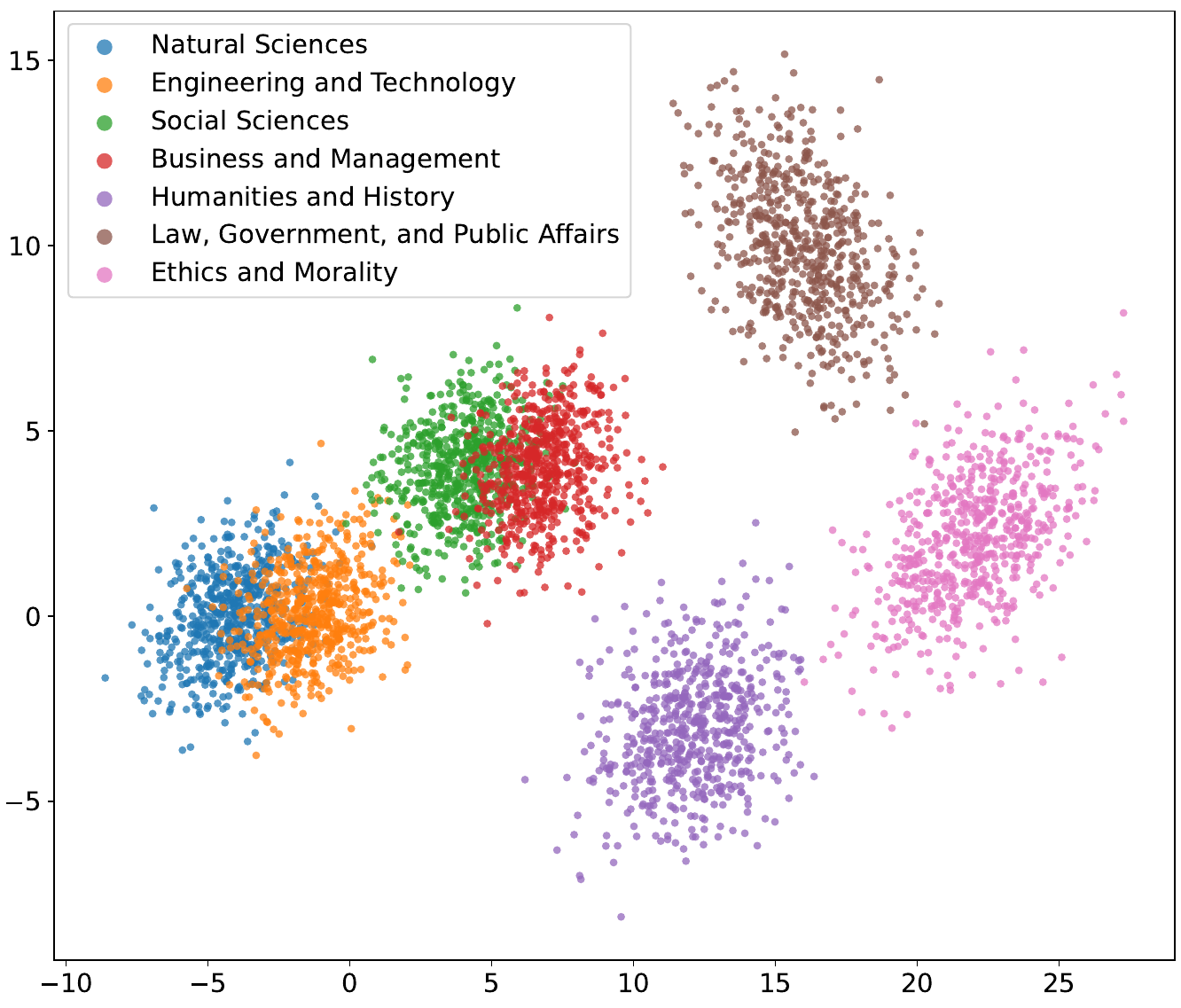}
\caption{t-SNE visualization of the encoder-induced latent space across the seven MMLU domains.}
\label{multi_domain_tsne}
\vspace{-1em}
\end{figure}

\begin{figure}[!t]
\centering
\includegraphics[width=9cm]{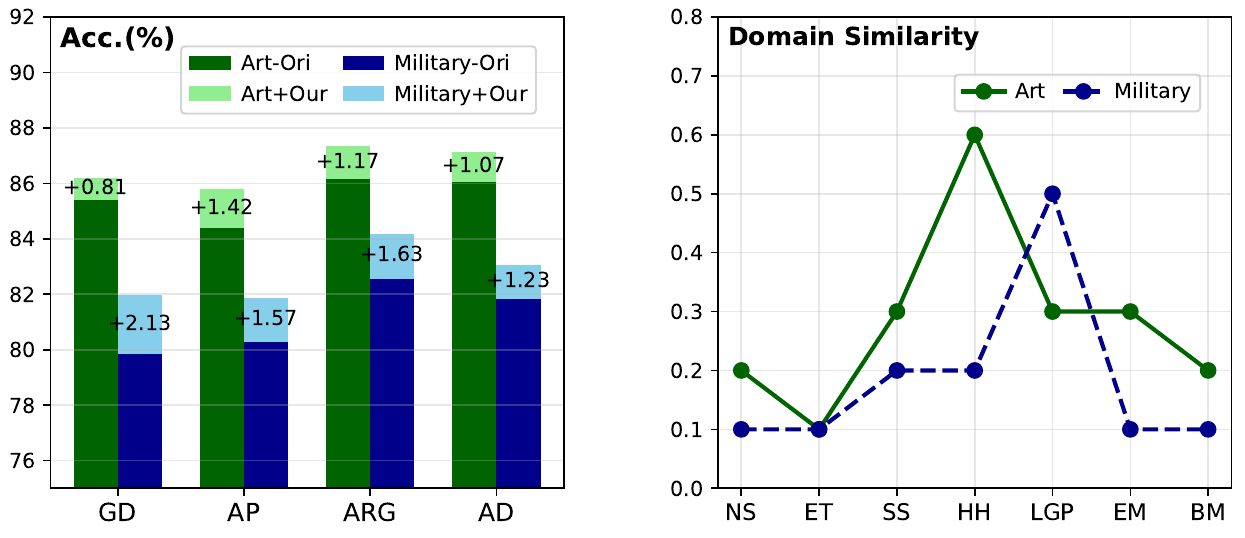}
\caption{Out-of-domain generalization on unseen domains. ``GD'', ``AP'', ``ARG'', and ``AD'' denote G-Designer, AgentPrune, ARG-Designer, and AgentDropout, respectively.}
\label{ood_domain}
\vspace{-1em}
\end{figure}

\subsection{Representation and Generalization Analysis}

\paragraph{Multi-Domain Latent Structure.}
To assess whether \texttt{TopoPrior} learns latent representations that are both reusable and domain-sensitive, we visualize the latent variable $z$ for all seven major MMLU domains. For each domain, we sample 400 examples, encode them into the latent space using the learned encoder, and project the resulting 128-dimensional representations to two dimensions using PCA followed by t-SNE~\cite{tsne_bib}. Figure~\ref{multi_domain_tsne} shows that the seven domains form several coherent clusters in the latent space. In particular, \emph{Natural Sciences} and \emph{Engineering \& Technology} exhibit partial overlap, likely due to shared quantitative and formal reasoning patterns, while \emph{Social Sciences} and \emph{Business \& Management} also overlap to some extent because of related economic and decision-oriented concepts. Other domains remain relatively distinct. Although this visualization is qualitative, it is consistent with the intended behavior of \texttt{TopoPrior}: the latent space captures reusable structural regularities while retaining domain-related variation that may be useful for topology initialization.

\paragraph{Out-of-Domain Generalization.}
To evaluate generalization beyond the training domains, we conduct experiments on two unseen domains that are not included in MMLU, namely Art~\cite{DBLP:conf/eccv/GarciaV18} and Military~\cite{DBLP:conf/coling/Zhu0ZXWKH24}, using Llama3-8B-Instruct. For a new query $q_{\mathrm{new}}$, we obtain a latent representation $z_{\mathrm{new}}\sim p'_\theta(z\mid q_{\mathrm{new}})$ from the learned prior network and compute its cosine similarity to the centroids of the seven in-domain latent clusters identified in Fig.~\ref{multi_domain_tsne}. Figure~\ref{ood_domain} shows that the learned prior maps unseen queries to semantically plausible regions of the latent space by associating them with related in-domain clusters, such as placing Art closer to \emph{Humanities \& History}. This behavior suggests that the topology initializer can produce plausible collaboration graphs even for unseen domains, which is consistent with the observed gains in out-of-domain performance.

\begin{table*}[!t]
\centering
\caption{Low-resource training results over MMLU of \texttt{TopoPrior} + ARG-Designer in terms of accuracy (\%).}
\begin{tabular}{cccccccccc}
\toprule
\multirow{2}{*}{\textbf{Domain}} & 
\multicolumn{2}{c}{\textbf{With \texttt{TopoPrior}?}} & 
\multirow{2}{*}{$\Delta$} & 
\multicolumn{2}{c}{\textbf{With \texttt{TopoPrior}?}} & 
\multirow{2}{*}{$\Delta$} & 
\multicolumn{2}{c}{\textbf{With \texttt{TopoPrior}?}} & 
\multirow{2}{*}{$\Delta$} \\
\cline{2-3} \cline{5-6} \cline{8-9}
& \textbf{No} & \textbf{Yes} & & \textbf{No} & \textbf{Yes} & & \textbf{No} & \textbf{Yes} & \\ \midrule
\multicolumn{1}{c}{\multirow{1}{*}{\textbf{Training data}}} & 
\multicolumn{3}{c}{\textbf{5\% of the original}} & 
\multicolumn{3}{c}{\textbf{10\% of the original}} & 
\multicolumn{3}{c}{\textbf{20\% of the original}} \\ \midrule
Natural Science & 60.34 & 61.95 & \textbf{1.61$\uparrow$} & 61.89 & 62.75 & \textbf{0.86$\uparrow$} & 62.64 & 64.96 & \textbf{2.32$\uparrow$} \\
Engineering \& Technology & 63.46 & 64.80 & \textbf{1.34$\uparrow$} & 64.72 & 66.05 & \textbf{1.33$\uparrow$} & 65.94 & 67.58 & \textbf{1.64$\uparrow$} \\
Social Sciences & 58.25 & 60.86 & \textbf{2.61$\uparrow$} & 60.13 & 63.31 & \textbf{3.18$\uparrow$} & 61.08 & 65.42 & \textbf{4.34$\uparrow$} \\
Humanities \& History & 62.88 & 63.92 & \textbf{1.04$\uparrow$} & 64.71 & 65.83 & \textbf{1.12$\uparrow$} & 66.15 & 67.89 & \textbf{1.74$\uparrow$} \\
Law, Government, Public Affairs & 55.13 & 57.64 & \textbf{2.51$\uparrow$} & 58.55 & 60.07 & \textbf{1.52$\uparrow$} & 60.41 & 63.32 & \textbf{2.91$\uparrow$} \\
Ethics \& Morality & 53.66 & 56.24 & \textbf{2.58$\uparrow$} & 56.85 & 58.70 & \textbf{1.85$\uparrow$} & 59.19 & 60.61 & \textbf{1.42$\uparrow$} \\
Business \& Management & 59.40 & 61.88 & \textbf{2.48$\uparrow$} & 61.23 & 63.65 & \textbf{2.42$\uparrow$} & 62.90 & 65.77 & \textbf{2.87$\uparrow$} \\
\midrule
\textbf{Average} & 59.02 & 61.04 & \textbf{2.02$\uparrow$} & 61.15 & 62.91 & \textbf{1.76$\uparrow$} & 62.62 & 65.08 & \textbf{2.46$\uparrow$} \\
\bottomrule
\end{tabular}
\label{few_shot}
\vspace{-1em}
\end{table*}

\subsection{Low-Resource Learning Analysis}
\label{low_resource}
To evaluate the robustness of \texttt{TopoPrior} in data-scarce domains, we conduct a low-resource analysis to assess whether \texttt{TopoPrior} can effectively leverage limited domain data. We compare the results of ARG-Designer with and without \texttt{TopoPrior} under several low-resource regimes on MMLU. Specifically, we subsample the original training data to 5\%, 10\%, and 20\% of the full size in each domain. The \texttt{TopoPrior} generator is trained using the same hyperparameters as in the full-data setting and subsequently evaluated on the corresponding test splits.

Table~\ref{few_shot} shows that \texttt{TopoPrior} improves the results of ARG-Designer across all data fractions, suggesting that it can transfer useful topological priors even under limited supervision. At the lowest data setting (5\%), incorporating \texttt{TopoPrior} yields particularly clear gains (e.g., +1.04 points in ``Humanities \& History'' and +2.58 points in ``Ethics \& Morality''), indicating that the learned prior can provide structural guidance when labeled examples are scarce. Despite fluctuations in absolute performance, \texttt{TopoPrior} maintains a positive improvement margin across all training-data sizes. These results suggest that the learned latent space remains useful in low-resource settings.

\subsection{Fixed Agent Pool Domain Generalization}
\label{fixed_pool_sec}

We next discuss why a fixed role pool may still generalize to niche and out-of-distribution (OOD) domains. The key intuition is that semantic similarity between queries and roles, together with the compositional flexibility of collaboration graphs, can enable role reuse beyond the domains explicitly represented in the training partition. We support this perspective with both conceptual discussion and empirical evidence.

As shown in Table~\ref{agent_roles}, the \texttt{TopoPrior} role pool spans broad domain categories (e.g., Natural Sciences, Engineering, and Social Sciences), each comprising multiple subdomains. Niche domains can often be associated with these broader categories. For example, a genetics query falls under Natural Sciences/Biology and can be addressed by both the ``Natural Science Expert'' and ``Medical Life Scientist'' roles. Crucially, \texttt{TopoPrior} does not assign roles statically; rather, it composes multiple roles and their interactions according to the content and context of each query. For a genetics question, the generated graph may include
\textit{(1)} a Natural Science Expert for core biological knowledge,
\textit{(2)} a Medical Life Scientist when the question involves genetic disorders,
\textit{(3)} a Chemistry Specialist for molecular genetics topics such as DNA structure, and
\textit{(4)} a General Coordinator to orchestrate collaboration.
Hence, even without an explicit ``Genetics Expert'' role, a combination of existing roles may still address the query effectively. The learned initializer adapts role composition to the query context, enabling a degree of generalization through semantic coverage.

\begin{table}[!t]
\centering
\caption{Performance on niche domains.}
\small
\setlength{\tabcolsep}{3pt}
\renewcommand{\arraystretch}{1.1}
\begin{tabular}{lccccc}
\toprule
\textbf{Model} & \multicolumn{2}{c}{\textbf{Astro-QA}} & \multicolumn{2}{c}{\textbf{TREC Genomics}} \\
\cmidrule(lr){2-3} \cmidrule(lr){4-5}
     & Acc. (\%) & Oracle Gap & Acc. (\%) & Oracle Gap \\
\midrule
ARG-Designer       & 71.4  & --3.7   & 61.8 & --3.3 \\
AutoGen            & 68.2  & --5.9   & 60.6 & --4.5 \\
Oracle             & 74.1  & --      & 65.1 & -- \\
\texttt{TopoPrior}+ARG       & \textbf{73.5} & \textbf{--0.6} & \textbf{64.3} & \textbf{--0.8} \\
\bottomrule
\end{tabular}
\label{niche_domains}
\vspace{-1em}
\end{table}

\subsubsection{Domain Validation of Generalization}
To further examine whether a fixed role pool can remain useful in niche domains, we construct two niche-domain test sets: Astro-QA~\cite{astro_qa} for astrophysics and TREC Genomics~\cite{Hersh2007TREC2G} for genetics. These domains are not explicitly represented in the original role pool. Specifically, astrophysics can be viewed as a subdomain of Natural Sciences/Physics, while genetics falls under Natural Sciences/Biology/Medicine. We compare three configurations: \textit{(1)} ARG-Designer and AutoGen~\cite{wu2024autogen}, which dynamically generate roles from scratch; \textit{(2)} \texttt{TopoPrior}+ARG, which uses the fixed role pool; and \textit{(3)} Oracle, in which a dedicated domain-specific role is added to the pool as a reference upper bound.
Table~\ref{niche_domains} shows that \texttt{TopoPrior}+ARG performs within 1\% of the Oracle, suggesting that semantic role composition can provide useful coverage for niche domains. By contrast, the dynamic role-generation baselines (ARG-Designer and AutoGen) perform worse in this setting, suggesting that reusing semantically related roles from a fixed pool may be more effective than generating roles from scratch.

\begin{table}[!t]
\caption{Token cost analysis for training and inference.}
\small
\setlength{\tabcolsep}{4pt}
\renewcommand{\arraystretch}{1.15}
\begin{tabular}{l r}
\toprule
\textbf{Metric} & \textbf{Value} \\
\midrule
Avg. tokens per training graph (ARG) & $\sim 1{,}200$ \\
Total training tokens ($\sim 100$k samples) & $\sim 120$M \\
\makecell[l]{Avg. inference tokens per test query \\ (ARG alone)} & $\sim 800$ \\
\makecell[l]{Avg. inference tokens per test query \\ (\texttt{TopoPrior}+ARG)} & $\sim 478$ \\
Token savings per test query & $322$ \\
Total inference savings (MMLU test set) & $4.83$M \\
Break-even test queries needed & $120$M $/$ $322 \approx 373{,}670$ \\
\bottomrule
\end{tabular}
\label{token_breakdown}
\flushleft
\vspace{-1em}
\end{table}

\begin{figure*}
\centering
\includegraphics[width=16cm]{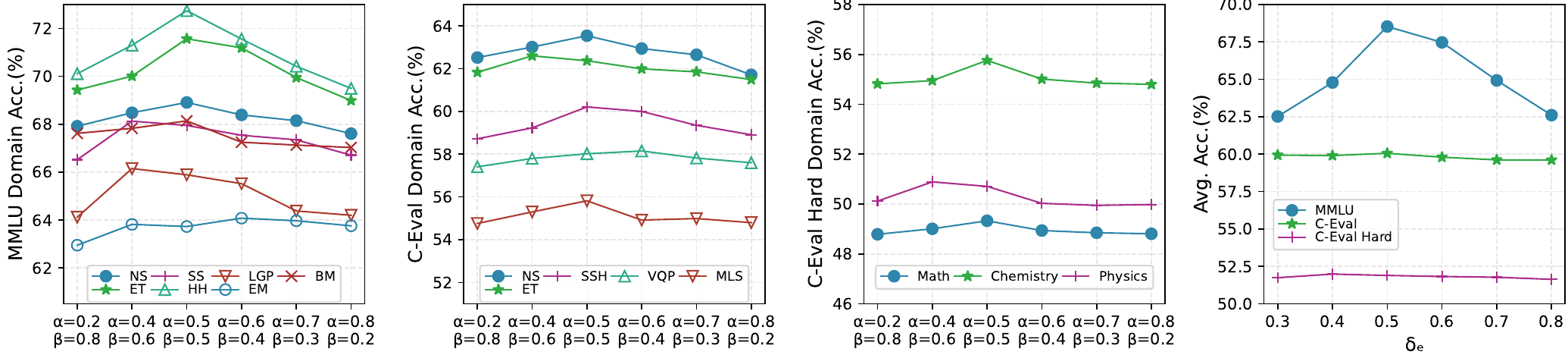}
\caption{Hyperparameter analysis of the loss coefficients $\alpha$, $\beta$, and edge generation threshold $\delta_e$. ``NS'', ``ET'', ``SS'', ``HH'', ``LGP'', ``EM'' and ``BM'' are abbreviations for different domains in MMLU and C-Eval, respectively.}
\label{hyper}
\vspace{-1em}
\end{figure*}

\subsection{Analysis of End-to-End Token Cost}
The one-time cost of generating reference collaboration graphs for training \texttt{TopoPrior} may be amortized by token savings during inference, which can make the overall pipeline cost-effective at sufficiently large deployment scales.
Reference graphs are generated only once during offline training. After training, the \texttt{TopoPrior} generator can produce initial collaboration topologies for any number of test queries without re-running the expensive topology-evolution algorithm. For any domain, the supervision cost scales with the training set size, whereas inference benefits from a \textbf{40.2\%} token reduction (see Section~\ref{detailed_analysis}) per test query. Token savings therefore accumulate as more test queries are processed and may eventually offset the initial offline cost.

We measure the total training token cost for graph generation on the MMLU training set (seven domains, $\sim$100k samples) using ARG-Designer, and compute total inference savings using \texttt{TopoPrior}+ARG on the MMLU test set (15k samples), extrapolating to larger test sets. As shown in Table~\ref{token_breakdown}, the MMLU test set alone does not recoup the training cost (4.83M saved tokens versus 120M tokens invested). However, at larger deployment scales, the break-even point is reached after approximately \textbf{373,670} queries. Once \texttt{TopoPrior} is trained, it can be reused across domains and datasets without additional graph generation, which further improves amortization in repeated-use scenarios.

\subsection{Hyperparameter Analysis}
\label{hyperparameters_analysis}
We analyze key hyperparameters, including the loss coefficients $\alpha$ and $\beta$, as well as the edge-generation threshold $\delta_e$, across multiple dataset subdomains using Llama3-8B-Instruct. Figure~\ref{hyper} reports the average performance over these settings. The best hyperparameter combination is used for the main experimental results.

\section{Conclusion}
\label{sec_conclusion}

In this paper, we presented \texttt{TopoPrior}, a framework for transferable topology prior learning for multi-agent LLM collaboration in multi-domain settings. Rather than constructing collaboration graphs from scratch for each task, \texttt{TopoPrior} learns reusable, query-conditioned topology priors from reference collaboration graphs and uses them to initialize downstream topology-evolution backbones. Experiments on MMLU, C-Eval, and additional out-of-domain settings show that \texttt{TopoPrior} improves downstream reasoning performance in the evaluated settings while reducing online communication cost, token usage, and refinement rounds. These results suggest that transferable topology-prior learning is a practical direction for improving the efficiency and adaptability of multi-agent LLM systems.

Several directions remain for future work. One is to reduce reliance on expensive offline supervision by using cheaper or more weakly supervised reference graphs. Another is to extend the current predefined role pool toward more open-ended role discovery and role composition. It is also promising to investigate tighter integration between topology-prior learning and downstream topology refinement in interactive reasoning settings. From a broader-impact perspective, \texttt{TopoPrior} may improve the efficiency and accessibility of multi-agent LLM systems by reducing communication overhead and improving collaboration quality. At the same time, such systems should be deployed with appropriate evaluation, transparency, and human oversight, especially in high-stakes applications.

\bibliographystyle{IEEEtran}  
\bibliography{reference}  

\newpage
\appendices
\section{Proofs of Theoretical Results}
\label{proof_theory}

In this appendix, we provide proofs and supporting statements for the analytical results in Section~\ref{sec:theory}. The analysis formalizes two core intuitions underlying \texttt{TopoPrior}: (i) latent-space alignment can support cross-domain generalization, and (ii) improved topology initialization can reduce the number of refinement rounds and the associated token cost.

Let $\mathcal{G}$ denote the space of directed labeled graphs corresponding to agent collaboration topologies. For any query $q$, domain label $d$, and ground-truth answer $y$, let $p^*(\cdot \mid q, d)$ denote an unknown high-utility distribution over collaboration graphs for that query and domain. Our learned generator is denoted by $p_\theta(\cdot \mid q)$. When executed with a graph $G$, the downstream multi-agent system produces a random output $\hat{y}$ with distribution $\pi_{\mathrm{MAS}}(\cdot \mid q, G)$. The task loss $\ell(\hat{y}, y) \in [0,1]$ measures the discrepancy between the system output and the ground truth.

We define the \emph{accuracy--token utility} as
\[
J_\lambda(G; q, y)
=
1
-
\mathbb{E}_{\hat{y}\sim\pi_{\mathrm{MAS}}(\cdot \mid q, G)}[\ell(\hat{y}, y)]
-
\lambda \widetilde{C}(G; q),
\]
where $\widetilde{C}(G; q) = C(G; q)/C_{\max}$ is the normalized token cost and $\lambda \ge 0$ is a trade-off coefficient between correctness and communication cost. Since $\ell(\hat{y}, y)\in[0,1]$ and $\widetilde{C}(G; q)\ge 0$, the utility $J_\lambda$ takes values in $[-\lambda,1]$.

\subsection{Proofs for Section~\ref{subsec:multidomain}: Cross-Domain Transfer via Latent Alignment}
\label{app:multidomain}

Let $Z$ be the latent space induced by the variational encoder. For each source domain $k \in \{1,\dots,K\}$, define the marginal latent distribution
\[
P_k(z)
=
\mathbb{E}_{(q,y)\sim\mathcal{D}_k}
\;
\mathbb{E}_{G^* \sim p^*(\cdot \mid q,k)}
\bigl[q_\phi(z \mid q, G^*)\bigr],
\]
where $q_\phi(z \mid q, G)$ denotes the encoder distribution. For a target domain $t$, we analogously define $P_t(z)$.

We consider a hypothesis class $\mathcal{H}$ of functions $h: Z \to \mathcal{Y}$. For a hypothesis $h$, the expected error on domain $d$ is
\[
\epsilon_d(h) = \mathbb{E}_{z \sim P_d}[\ell_h(z)],
\]
where $\ell_h$ is a bounded loss.

To quantify discrepancy between latent distributions, we use the $\mathcal{H}\Delta\mathcal{H}$-divergence.

\begin{definition}[$\mathcal{H}\Delta\mathcal{H}$-divergence]
\label{def:hdiv}
For two distributions $P$ and $Q$ on $Z$,
\begin{multline*}
d_{\mathcal{H}\Delta\mathcal{H}}(P,Q)
=
2\sup_{h,h'\in\mathcal{H}}
\Bigl|
\Pr_{z\sim P}[h(z)\neq h'(z)] \\
- \Pr_{z\sim Q}[h(z)\neq h'(z)]
\Bigr|.
\end{multline*}
\end{definition}

This divergence measures the largest change in pairwise hypothesis disagreement between the two distributions.

\begin{proof}[Proof of Theorem~\ref{thm:multi_domain}]
Let $\alpha = (\alpha_1,\dots,\alpha_K)$ be convex weights such that $\alpha_k \ge 0$ and $\sum_{k=1}^K \alpha_k = 1$. Define the source-mixture distribution
\[
P_\alpha = \sum_{k=1}^K \alpha_k P_k
\]
and the corresponding weighted source error
\[
\epsilon_\alpha(h) = \sum_{k=1}^K \alpha_k \epsilon_k(h).
\]

Let
\[
h^* \in \arg\min_{h\in\mathcal{H}} \bigl(\epsilon_\alpha(h) + \epsilon_t(h)\bigr),
\]
and define
\[
\lambda_{\alpha,t}^* = \epsilon_\alpha(h^*) + \epsilon_t(h^*).
\]

Starting from
\[
\epsilon_t(h)
=
\epsilon_\alpha(h)
+
\bigl(\epsilon_t(h)-\epsilon_t(h^*)\bigr)
+
\bigl(\epsilon_t(h^*)+\epsilon_\alpha(h^*)\bigr)
+
\bigl(\epsilon_\alpha(h^*)-\epsilon_\alpha(h)\bigr),
\]
we obtain
\[
\epsilon_t(h)
\le
\epsilon_\alpha(h)
+
\left|
\epsilon_t(h)-\epsilon_t(h^*)
-
\bigl(\epsilon_\alpha(h)-\epsilon_\alpha(h^*)\bigr)
\right|
+
\lambda_{\alpha,t}^*.
\]

Using the standard multi-source domain adaptation bound based on $\mathcal{H}\Delta\mathcal{H}$, the absolute term is bounded by
\[
\frac{1}{2}d_{\mathcal{H}\Delta\mathcal{H}}(P_\alpha,P_t).
\]
Therefore,
\[
\epsilon_t(h)
\le
\epsilon_\alpha(h)
+
\frac{1}{2}d_{\mathcal{H}\Delta\mathcal{H}}(P_\alpha,P_t)
+
\lambda_{\alpha,t}^*.
\]
Substituting $\epsilon_\alpha(h)=\sum_{k=1}^K \alpha_k \epsilon_k(h)$ concludes the proof.
\end{proof}

\paragraph{Interpretation.}
Theorem~\ref{thm:multi_domain} bounds the target-domain error by three terms: the weighted source-domain error, the latent-space divergence between the source mixture and the target domain, and the best joint error achievable by the hypothesis class. In \texttt{TopoPrior}, the adversarial regularizer is intended to reduce latent-space domain discrimination, thereby lowering divergence and supporting transfer to unseen domains.

\subsection{Proofs for Section~\ref{subsec:acceleration}: Topology Initialization as Search Acceleration}
\label{app:acceleration}

We now analyze how improved initialization can reduce refinement rounds and token cost. Let
\[
U_t = \mathbb{E}[J_\lambda(G_t; q, y)]
\]
be the expected utility at evolution step $t$, starting from $G_0$, and let $U^*$ denote the optimal utility achievable by the evolution procedure.

\begin{proof}[Proof of Theorem~\ref{thm:rounds}]
Under Assumption~\ref{asm:contraction}, for all $t$,
\[
U^* - U_{t+1} \le (1-\eta)(U^* - U_t).
\]
Recursively applying this inequality for $T$ steps gives
\[
U^* - U_T \le (1-\eta)^T (U^* - U_0).
\]
To ensure $U^* - U_T \le \epsilon$, it suffices to require
\[
(1-\eta)^T (U^* - U_0) \le \epsilon,
\]
which is equivalent to
\[
T \ge \frac{\log\bigl((U^* - U_0)/\epsilon\bigr)}{\log\bigl(1/(1-\eta)\bigr)},
\]
since $0<1-\eta<1$. This completes the proof.
\end{proof}

\begin{proof}[Proof of Corollary~\ref{cor:rounds_reduction}]
Applying Theorem~\ref{thm:rounds} to prior-based and scratch initialization gives
\[
T_{\mathrm{prior}}(\epsilon)
=
\frac{\log\bigl((U^* - U_0^{\mathrm{prior}})/\epsilon\bigr)}
{\log(1/(1-\eta))},
\]
\[
T_{\mathrm{scratch}}(\epsilon)
=
\frac{\log\bigl((U^* - U_0^{\mathrm{scratch}})/\epsilon\bigr)}
{\log(1/(1-\eta))}.
\]
Subtracting yields
\[
T_{\mathrm{prior}}(\epsilon) - T_{\mathrm{scratch}}(\epsilon)
=
\frac{\log\!\left(\frac{U^* - U_0^{\mathrm{prior}}}{U^* - U_0^{\mathrm{scratch}}}\right)}
{\log(1/(1-\eta))}.
\]
Since $U_0^{\mathrm{prior}} > U_0^{\mathrm{scratch}}$, the numerator is negative and the denominator is positive, so $T_{\mathrm{prior}}(\epsilon) < T_{\mathrm{scratch}}(\epsilon)$.
\end{proof}

Next, we formalize token-cost savings from fewer refinement rounds and sparser initial graphs.

\begin{assumption}[Token cost bounded by graph size]
\label{asm:token_size}
There exist constants $a_V, a_E, b \ge 0$ such that for any graph $G$ and query $q$,
\[
C(G; q) \le a_V |V(G)| + a_E |E(G)| + b.
\]
\end{assumption}

\begin{assumption}[Bounded graph size during evolution]
\label{asm:graph_size}
For all steps $t$ before termination,
\[
\mathbb{E}[|V(G_t)| + |E(G_t)|] \le M
\]
for some constant $M \ge 1$.
\end{assumption}

\begin{proposition}[Total expected token cost over $T$ rounds]
\label{prop:token_total}
Under Assumptions~\ref{asm:token_size} and~\ref{asm:graph_size},
\[
\mathbb{E}\!\left[\sum_{t=0}^{T-1} C(G_t; q)\right]
\le
T\bigl((a_V+a_E)M + b\bigr).
\]
\end{proposition}

\begin{proof}
From Assumption~\ref{asm:token_size}, we have
\[
C(G_t; q) \le a_V |V(G_t)| + a_E |E(G_t)| + b.
\]
Taking expectations on both sides and using Assumption~\ref{asm:graph_size}, we obtain
\[
\mathbb{E}[C(G_t; q)]
\le
a_V \mathbb{E}[|V(G_t)|] + a_E \mathbb{E}[|E(G_t)|] + b
\le
(a_V+a_E)M + b.
\]
Summing over $t=0,\dots,T-1$ establishes the result.
\end{proof}

\begin{corollary}[Token savings from fewer rounds]
\label{cor:token_savings}
If \texttt{TopoPrior} reduces the number of rounds from $T_{\mathrm{scratch}}$ to $T_{\mathrm{prior}}$, then the reduction in the bound of Proposition~\ref{prop:token_total} is
\[
\bigl(T_{\mathrm{scratch}} - T_{\mathrm{prior}}\bigr)\bigl((a_V+a_E)M + b\bigr).
\]
\end{corollary}

\begin{assumption}[Token cost monotonicity in graph size]
\label{asm:token_mono}
For fixed $q$, $C(G; q)$ is non-decreasing in both $|V(G)|$ and $|E(G)|$.
\end{assumption}

\begin{proposition}[Sparser initialization reduces first-round cost]
\label{prop:sparsity}
Under Assumption~\ref{asm:token_mono}, if the prior-initialized graph $G_0^{\mathrm{prior}}$ satisfies
\[
\mathbb{E}[|V(G_0^{\mathrm{prior}})|] \le \mathbb{E}[|V(G_0^{\mathrm{scratch}})|],
\]
\[
\mathbb{E}[|E(G_0^{\mathrm{prior}})|] \le \mathbb{E}[|E(G_0^{\mathrm{scratch}})|],
\]
then
\[
\mathbb{E}[C(G_0^{\mathrm{prior}}; q)]
\le
\mathbb{E}[C(G_0^{\mathrm{scratch}}; q)].
\]
\end{proposition}

\begin{proof}
By Assumption~\ref{asm:token_mono}, smaller or equal node and edge counts imply no greater token cost. Taking expectations preserves the inequality.
\end{proof}

\paragraph{Interpretation.}
These results highlight two complementary efficiency effects. First, under Assumption~\ref{asm:contraction}, better initialization reduces the number of refinement rounds needed to achieve a target utility. Second, if the initialized graph is also sparser, the per-round token cost decreases. Together, these results provide analytical support for the efficiency trends observed in our experiments.

\end{document}